\documentclass[10pt]{article}

\usepackage[letterpaper,textwidth=5.5in,textheight=9in,top=1in,headheight=12pt,headsep=25pt,footskip=30pt]{geometry}
\usepackage[T1]{fontenc}
\usepackage[utf8]{inputenc}
\usepackage{lmodern}
\usepackage{microtype}
\usepackage{enumitem}
\usepackage{float}
\usepackage{xr-hyper}
\usepackage{longtable}
\usepackage{booktabs}
\usepackage{amsmath,amssymb,amsfonts,amsthm,bm}
\usepackage{mathtools}
\usepackage{bbm}
\usepackage{nicefrac}
\usepackage{graphicx}
\usepackage{wrapfig}
\usepackage{subcaption}
\usepackage{multirow}
\usepackage{makecell}
\usepackage[dvipsnames,table]{xcolor}
\usepackage[most]{tcolorbox}
\usepackage[ruled,vlined]{algorithm2e}
\usepackage{xspace}
\usepackage{tikz}
\usetikzlibrary{arrows.meta,positioning}
\usepackage{pifont}
\usepackage{titlesec}
\usepackage{tabularx}

\titlespacing*{\section}{0pt}{2.0ex plus .5ex minus .2ex}{0.8ex}
\titlespacing*{\subsection}{0pt}{1.6ex plus .4ex minus .2ex}{0.6ex}
\titlespacing*{\subsubsection}{0pt}{1.3ex plus .3ex minus .2ex}{0.5ex}
\setlist{nosep}

\definecolor{lightgray}{gray}{0.8}
\definecolor{boxborder}{HTML}{215F92}
\definecolor{boxback}{HTML}{E8F1F8}
\definecolor{dircolor}{HTML}{0072B2}
\definecolor{indcolor}{HTML}{E67E22}
\definecolor{totcolor}{HTML}{009E73}

\theoremstyle{plain}
\newtheorem{theorem}{Theorem}[section]
\newtheorem{proposition}[theorem]{Proposition}
\newtheorem{lemma}[theorem]{Lemma}
\newtheorem{corollary}[theorem]{Corollary}
\theoremstyle{definition}

\newtheorem{assumption}[theorem]{Assumption}
\theoremstyle{remark}
\newtheorem{remark}[theorem]{Remark}

\newtheorem{example}{Example}[section]

\newcommand{\fna}{\text{FNA}}

\newcommand{\E}{\mathbb{E}}

\newcommand*\diff{\mathop{}\!\mathrm{d}}

\newcommand*\circledgreen[1]{%
\tikz[baseline=(char.base)]{\node[shape=circle,draw=ForestGreen!60,fill=ForestGreen!10,thick,inner sep=1pt,font=\upshape] (char) {\scriptsize\textsf{#1}};}}

\newcommand*\circledblue[1]{%
\tikz[baseline=(char.base)]{\node[shape=circle,draw=NavyBlue!60,fill=NavyBlue!10,thick,inner sep=1pt] (char) {\scriptsize\textsf{#1}};}}

\newcommand{\dir}[1]{\textcolor{dircolor}{#1}}
\newcommand{\ind}[1]{\textcolor{indcolor}{#1}}
\newcommand{\tot}[1]{\textcolor{totcolor}{#1}}
\newcommand{\cmark}{\textcolor{teal}{\ding{51}}}
\newcommand{\xmark}{\textcolor{red!70!black}{\ding{55}}}

\makeatletter
\def\maketag@@@#1{\hbox{\m@th\normalfont\normalsize#1}}
\makeatother

\usepackage[backend=biber,style=numeric,natbib=true,doi=false,isbn=false,url=false,eprint=false,sorting=none]{biblatex}
\usepackage[colorlinks=true,linkcolor=NavyBlue,citecolor=NavyBlue,urlcolor=NavyBlue]{hyperref}

\title{Path-specific harm decomposition: A partial identification framework}

\author{%
  Ruizi Yan$^{1}$,
  Dennis Frauen$^{2,3}$,
  Maresa Schr{\"o}der$^{2,3}$,
  Stefan Feuerriegel$^{2,3}$\\[0.5em]
  $^{1}$Department of Statistics, University of Oxford\\
  $^{2}$LMU Munich\\
  $^{3}$Munich Center for Machine Learning (MCML)\\[0.5em]
}
\date{}

\begin{document}

\maketitle

\begin{abstract}
A central goal when designing treatment policies is often to \textit{``do no harm''}, that is, to avoid interventions that improve average outcomes while worsening outcomes for some individuals. A widely used notion for harm is the fraction of negatively affected (FNA), defined as the probability that an intervention decreases an individual's outcome. However, in many applications, treatments operate through mediators, and a single ``total'' FNA can obscure whether harm arises primarily through direct pathways or indirect (mediator-induced) pathways. In this work, we introduce a path-specific analogue of the FNA. For this, we disentangle total harm into direct and indirect harm in causal mediation settings. However, these quantities depend on joint distributions of potential outcomes that are not point-identified even in randomised controlled trials. As a remedy, we develop a novel partial identification framework for direct and indirect FNA. In our framework, 
we (i)~derive sharp Makarov bounds for the FNA, and (ii)~propose a semiparametrically efficient estimator with valid confidence intervals for these bounds under mild margin conditions. We demonstrate our framework across various numerical experiments. To the best of our knowledge, we are the first to study path-specific decomposition of causal harm and to develop an orthogonal inference framework for its analysis.
\end{abstract}

\section{Introduction}

The principle of \textit{``do no harm''} guides decision-making in fields like medicine~\citep{Lilienfeld2007PsychologicalTT,beauchamp2019principles,page2012fourprinciples} and social policy~\citep{sandvik2017donoharm}. In many applications, treatment policies are evaluated based on their average effects, yet an intervention that improves outcomes on average may still harm a subset of individuals~\citep{kravitz2004evidence,iomc2024precision}.

\textbf{Illustrative examples.} $\bullet$\,Consider a cancer treatment such as chemotherapy for lung cancer that improves survival on average but exposes
some patients to severe toxicity; for instance, in the KEYNOTE-189 trial, pembrolizumab plus chemotherapy improved survival in metastatic non-squamous non-small-cell lung cancer, but grade-3-or-higher adverse events occurred in $67.2\%$ of patients in the combination arm compared with $65.8\%$ in the placebo group \citep{gandhi2018pembrolizumab}. Here, optimizing treatment decisions based solely on average outcomes may therefore overlook patients who are harmed~\citep{Weberpals2025cmlinoncology}. $\bullet$\,Consider an education policy that improves average test scores but fails to reach or even disadvantage certain subgroups due to unequal access or differences in how the policy is implemented~\citep{oxman2024adverseeducation}. These examples highlight that optimizing policies based on overall averages can mask harm to a nontrivial fraction of individuals.

Prior work has formalized the above principle of ``do no harm'' through a causal notion of harm ~\citep{richens2022counterfactualharm,vaskov2024noharmRL,Straitouri2024controllingCF}. An individual is said to be harmed if their outcome under treatment is worse than under the alternative treatment. Hence, harm can be measured at the population level through the \textit{\textbf{fraction negatively affected (FNA)}}, which is defined as the probability that an intervention worsens an individual’s outcome~\citep{Kallus2022FNA_sharp_bounds,pena2023boundingprobabilitiesbenefitharm}.

\emph{Why FNA rather than a conditional average effect (CATE)?} An average effect compares marginal mean outcomes under the two interventions. In contrast, FNA asks a different question: what fraction of the \emph{same} individuals are made worse off by the intervention? Conditioning on an observed subgroup $Z=z$ can reveal heterogeneous average effects through a CATE, but it remains a comparison of two marginal means within the subgroup and does not identify the within-unit harm event.

Existing work on the FNA focuses on a single notion of \textit{total} harm \citep{Kallus2022FNA_sharp_bounds}, which captures whether an intervention worsens outcomes overall, but it does not distinguish \textit{how} harm arises. In particular, it does not distinguish between harm occurring through direct or indirect (mediator-induced) pathways. As a result, a single aggregate measure of harm may obscure important heterogeneity in how individuals are affected.

\emph{Why path-specific harm matters:} Understanding path-specific harm is important because it reveals whether harm arises through direct effects of a treatment or indirectly through mediating mechanisms. Treatment effects often operate through both pathways. For example, 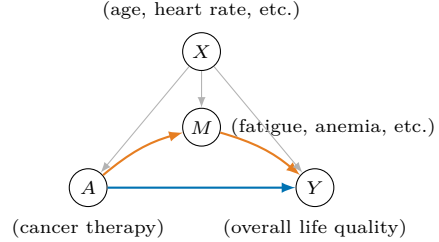
\begin{wrapfigure}{r}{0.4\textwidth}
\vspace{-10pt}
\centering
\begin{tikzpicture}[
    dag/.style={draw, circle, minimum size=0.5cm, inner sep=1pt, font=\scriptsize},
    ex/.style={draw=none, font=\scriptsize, align=center},
    >=latex
]

\node[dag] (X) at (0,1.8) {$X$};
\node[dag] (A) at (-1.5,0) {$A$};
\node[dag] (Y) at (1.5,0) {$Y$};
\node[dag] (M) at (0,0.8) {$M$};

\node[ex, above=1pt of X] {(age, heart rate, etc.)};
\node[ex, below=1pt of A] {(cancer therapy)};
\node[ex, below=1pt of Y] {(overall life quality)};
\node[ex, right=0.1pt of M] {(fatigue, anemia, etc.)};

\draw[->, thick, color=dircolor] (A) -- (Y);

\draw[->, thick, color=indcolor] (A) to[bend left=12] (M);
\draw[->, thick, color=indcolor] (M) to[bend left=12] (Y);

\draw[->, thin, gray!60] (X) -- (A);
\draw[->, thin, gray!60] (X) -- (Y);
\draw[->, thin, gray!60] (X) -- (M);

\end{tikzpicture}
\vspace{-6pt}
\caption{\footnotesize
Causal diagram with treatment $A$, mediator $M$, outcome $Y$, and covariates $X$.
Blue denotes the direct path $\dir{A\!\to\!Y}$ (\dir{direct} effect of cancer therapy), and magenta the indirect path
$\ind{A\!\to\!M\!\to\!Y}$ (likely \ind{indirect} harmful impact to the outcome through the adverse side effects).} 
\label{fig:casualdag}
\vspace{-10pt}
\end{wrapfigure}
a cancer treatment may improve survival through its therapeutic effect while simultaneously harming patients through severe side effects (Figure~\ref{fig:casualdag}).
Distinguishing these pathways helps identify the sources of harm and how they differ across individuals. More generally, different pathways can induce opposing effects, so that focusing only on total harm may mask pathway-specific risks that are crucial for decision-making.

\textbf{Our framework:} Here, we develop path-specific analogues of the FNA to capture direct and indirect sources of harm. Formally, we define direct and indirect FNA using nested potential outcomes, where the indirect component involves cross-world quantities of the form $Y(a, M(a'))$ for $a\neq a'$. These quantities are inherently unobservable, even in randomised controlled trials (RCTs), because they depend on joint distributions of potential outcomes. As a result, path-specific FNA is not point identifiable, which makes both the formalization and estimation non-trivial.

To address this, we develop a partial identification framework. Figure~\ref{fig:overview} provides an overview. Formally, 
we define path-specific FNA through comparisons of nested
potential outcomes and derive sharp Makarov bounds. The bounds characterize the tightest possible range of each FNA given the identified marginals. Then, we develop semiparametrically efficient, orthogonal estimators for these bounds. Our estimators are flexible and can be combined with arbitrary machine learning models (e.g., neural networks) for nuisance estimation, and the estimators further introduce debiasing to control estimation error. We demonstrate across a range of numerical experiments that our proposed estimators substantially improve over na{\"i}ve plug-in approaches, especially in finite samples.

\textbf{Contributions:} Our main contributions are three-fold:\footnote{Code is available at https://github.com/ruiziyan9/path-specific-causal-harm. 
} \circledblue{1}~We define path-specific analogues of FNA. \circledblue{2}~We develop a partial identification framework and, for this, derive conditional sharp bounds for the marginal (in)direct FNA. \circledblue{3}~We develop semiparametrically efficient estimators for the (in)direct FNA bounds.

\begin{figure}[htbp]
    \centering
    \includegraphics[width=\linewidth]{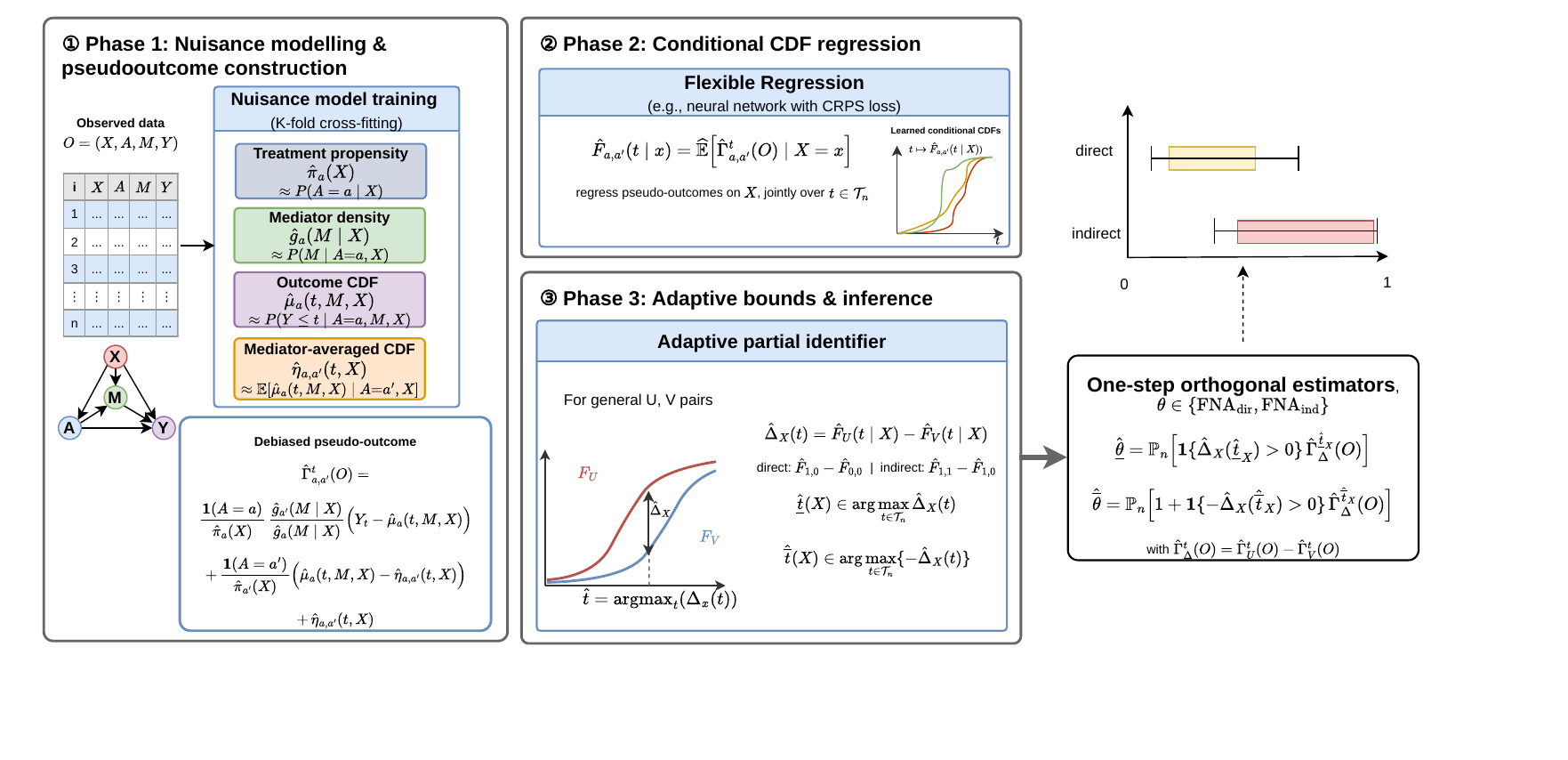}
    \vspace{-1.5cm}
    \caption{\textbf{Overview} of our semiparametric framework for estimating path-specific FNA (i.e., bounds on the direct, indirect, and total FNA).}
    \label{fig:overview}
    \vspace{-0.5cm}
\end{figure}

\section{Related Work}

Below, we review key literature streams; an extended related work is in Appendix~\ref{app:extendedrw}.

$\bullet$\,\textbf{Causal mediation analysis.} Causal mediation analysis decomposes treatment effects into direct and indirect effects \citep{imai2010mediationgeneral}. More generally, path-specific effects can formalize different causal pathways using nested counterfactuals such as $Y(a,M(a'))$ \citep{avin2005identifiability}. Later works have developed robust estimators and semi-parametric efficiency theory for this setting \citep{tchetgen2012semiparametric,Zheng2012TMLE,Farbmacher2022mediationdml}. $\Rightarrow$ \emph{However, this literature stream focuses on average effects, while we study individual-level harm.}

$\bullet$\,\textbf{Causal fairness.} Causal fairness is related to notions of harm, but with important differences. Causal fairness studies \textit{group-level} disparities across causal pathways \citep{kilbertus2017avoiding,zhang2018causalfairness,kusner2017counterfactual,imai2023principalfairness}. Some works also focus on path-specific notions of fairness \citep{chikahara2021fair,liu2023path,yao2023pathspecificfairapplication}. $\Rightarrow$ \emph{However, fairness targets disparities at the group level, while harm in the FNA notion measures adverse effect at the individual level.\footnote{We study whether a treatment makes the same individual worse off than under the alternative, and then aggregate this individual-level event into an FNA.} Further, FNA depends on unobserved joint potential outcomes and is therefore only partially identifiable in our setting, unlike most fairness criteria.}

$\bullet$\,\textbf{Harm notions.} Harm has long been studied in philosophy, where it is defined via \emph{actual causality}: an action harms an individual if it causes a worse outcome in a specific realized scenario~\citep{halpern2016actualcausality} and is thus often studied through utility-based comparisons~\citep{beckers2022actualcausalharm,beckers2023quantifyingharm}. More broadly, the literature can be grouped into three streams. (1) \textbf{Interventional} notions define harm through average effects, e.g., negative ATE~\citep{sarvet2025perspectives}. (2) \textbf{Counterfactual} notions, such as the probability of necessity, are located in layer~3 of Pearl's causal ladder~\citep{pearl1995ladder, Pearl1999probcausation} and assess, in a retrospective sense, whether a realized negative outcome was caused by the treatment~\citep{richens2022counterfactualharm,Straitouri2024controllingCF}. (3) \textbf{FNA} measures, in a prospective sense, how often individuals are made worse off, captured by the fraction negatively affected~\citep{Mueller2022individualeffect}. $\Rightarrow$ \emph{Our work is located in stream (3).}

Several recent works study the FNA as a measure of individual-level harm \citep{Kallus2022FNA_sharp_bounds, BenMichael2024assymmetric}. These works develop sharp bounds and robust estimators under partial identification, since the joint distribution of potential outcomes is unobserved even in randomised experiments \citep{heckman1997heterogeneity, Fan2010covariatessharpbounds}. $\Rightarrow$ \emph{However, existing work focuses on \textbf{total} harm and does \textbf{not} distinguish how harm arises across different causal pathways. Our work builds on this line of research and, for the first time, decomposes FNA into direct and indirect components.}

\begin{table}
\centering
\caption{\textbf{Overview of related work on modelling harm.} 
}
\label{tab:related_work_harm}
\footnotesize
\resizebox{\textwidth}{!}{%
\begin{tabular}{llllcccc}
\toprule
\textbf{Work} & \textbf{Harm notion} & \textbf{Level} & \textbf{Path-specific} & \textbf{Identification} & \textbf{Orthogonal} & \textbf{Cov.-assisted} \\
\midrule
\citet{beckers2023quantifyingharm}           & Actual causality  & event & \xmark & --- & \xmark & \xmark \\
\midrule
\citet{sarvet2025perspectives}              & Negative ATE                                & pop.\ & \xmark & point   & \xmark & \xmark \\
\midrule
\citet{richens2022counterfactualharm}       & Probability of necessity     & pop.\ & \cmark & partial & \xmark & \xmark \\
\citet{heckman1997heterogeneity}            & Total FNA                   & pop.\ & \xmark & partial & \xmark & \xmark \\
\citet{Fan2010covariatessharpbounds}        & Total FNA                   & pop.\ & \xmark & partial & \xmark & \cmark \\
\citet{Kallus2022FNA_sharp_bounds}          & Total FNA                   & pop.\ & \xmark & partial & \cmark & \cmark \\
\citet{BenMichael2024assymmetric}           & Total FNA                           & pop.\ & \xmark & partial & \cmark & \xmark \\
\midrule
\citet{pearl2001direct}                     & Negative NDE \& NIE             & pop.\ & \cmark & point   & \xmark & \xmark \\
\citet{imai2010mediationgeneral}            & Negative NDE \& NIE             & pop.\ & \cmark & point   & \xmark & \xmark \\
\citet{tchetgen2012semiparametric}          & Negative NDE \& NIE          & pop.\ & \cmark & point   & \cmark & \xmark \\
\citet{Farbmacher2022mediationdml}          & Negative NDE \& NIE                   & pop.\ & \cmark & point   & \cmark & \xmark \\
\midrule
\textbf{Our work}  & \textbf{Direct \& indirect FNA} & \textbf{pop.} & \cmark & partial & \cmark & \cmark \\
\bottomrule
\multicolumn{7}{p{1.4\textwidth}}{
\emph{Level}: whether the harm notion applies to a specific realized event (\emph{event}) or a population (\emph{pop.}).
\emph{Path-specific}: decomposes harm into direct and indirect components via a mediator. 
\emph{Identification}: \emph{point} = point-identified from observed data under standard assumptions; \emph{partial} = only bounds are available; \emph{---} = notion applies to a single past event, so population-level identification is not the relevant question.
\emph{Orthogonal}: Neyman-orthogonal / multiply robust estimator with valid ML inference. 
\emph{Cov.-assisted}: tighter bounds obtained by conditioning on covariates.
\emph{NDE/NIE}: natural direct/indirect effect.
}
\end{tabular}}%
\vspace{-0.4cm}
\end{table}

\textbf{Research gap.} To the best of our knowledge, we are the first to study path-specific harm and propose orthogonal learners for estimating Makarov bounds for direct and indirect FNA.

\section{Problem setup}

\textbf{Notation.} We use capital letters  (e.g., $Y$) to denote random variables and lower-case letters  (e.g., $y$) to denote their realizations. Expectations are taken with respect to the underlying data-generating distribution unless otherwise specified. Let $\pi_a(x) := \Pr(A=a\mid X=x)$ denote the treatment propensity score and let $g_a(m\mid x) := f_{M\mid A=a,X}(m\mid x)$ denote the mediation propensity.
For any threshold $t\in\mathbb R$, we define the indicator outcome $Y_t := \mathbbm 1\{Y\le t\}$ and the conditional outcome expectation $\mu_a(t,m,x) := \mathbb E\!\big[Y_t \mid A=a,M=m,X=x\big]$.
Define the mediator-averaged conditional cumulative distribution functional (CDF)
\begin{equation}
\eta_{a,a'}(t,x) := \int \mu_a(t,m,x)\, g_{a'}(m\mid x)\, dm,
\end{equation}
(where the integral is a sum if $M$ is discrete). The functional $\eta_{a,a'}(t,x)$ represents the conditional CDF of the nested potential outcome $Y(a,M(a'))$ given covariates $X=x$. The corresponding marginal nested CDF is
\begin{equation}
F_{a,a'}(t) := \Pr\!\big(Y(a,M(a'))\le t\big).
\end{equation}

\textbf{Observed data.} 
We observe i.i.d.\ samples $O=(X,A,M,Y)$, under the causal diagram from Fig~\ref{fig:casualdag}, where $A\in\{0,1\}$ is a binary treatment, $M$ is a mediator, $Y$ is an outcome, and $X$ denotes baseline covariates. Throughout, we interpret ``negative'' as \emph{lower} outcomes being worse; i.e., harm corresponds to a decrease in $Y$.

\textbf{Nested potential outcomes.}
We formulate the problem based on the Neyman-Rubin potential outcome (PO) framework~\citep{Rubin2025PO}. For each unit, define a collection of \emph{potential outcomes} $\{Y(a) : a \in \mathcal{A}\}$, where $Y(a)$ denotes the outcome that would be observed if, possibly contrary to fact, the treatment were set to $A = a$.
The observed outcome satisfies the \emph{consistency} relation, $Y = Y(A)$.
In the presence of a mediator $M$, define mediator potential outcomes 
$\{M(a) : a \in \mathcal{A}\}$ and nested potential outcomes 
$\{Y(a,m) : a \in \mathcal{A}, m \in \mathcal{M}\}$, where $Y(a,m)$ denotes 
the outcome under treatment $a$ and mediator set to $m$. The nested potential 
outcome is defined as $Y(a, M(a'))$.

\begin{tcolorbox}[
    colback=boxback, 
    colframe=boxborder, 
    arc=8pt,           
    outer arc=8pt, 
    boxrule=0.8pt,     
    left=10pt,         
    right=10pt,        
    top=5pt,          
    bottom=5pt        
]
\textbf{\dir{Direct}, \ind{indirect}, and \tot{total} FNA.} Each FNA corresponds to the probability that one nested potential outcome is smaller than another, given by
\begin{align}
\dir{\fna_{\mathrm{dir}}} 
&:= \dir{\Pr\!\big\{Y(1,M(0)) < Y(0,M(0))\big\}}, \\
\ind{\fna_{\mathrm{ind}}} 
&:= \ind{\Pr\!\big\{Y(1,M(1)) < Y(1,M(0))\big\}}, \\
\tot{\fna_{\mathrm{tot}}} 
&:= \tot{\Pr\!\big\{Y(1,M(1)) < Y(0,M(0))\big\}}.
\end{align}
\end{tcolorbox}
For intuition, consider a new drug with treatment $A$, treatment-induced toxicity as mediator $M$, and overall health as outcome $Y$, where larger $Y$ is better. Then $Y(1,M(0))<Y(0,M(0))$ means that the drug would harm the patient even if toxicity were held at the level it would have taken under standard care. In contrast, $Y(1,M(1))<Y(1,M(0))$ means that the drug-induced change in toxicity itself worsens the patient's health.

Note that the FNAs are not linear and thus nonadditive, that is, $\tot{\fna_{\mathrm{tot}} }\neq \dir{\fna_{\mathrm{dir}}} +\ind{\fna_{\mathrm{ind}} }$. To see why, consider the two-step counterfactual trajectory $Y(0,M(0))\rightarrow Y(1,M(0))\rightarrow Y(1,M(1))$. Direct FNA asks whether the first step is harmful, indirect FNA asks whether the second step is harmful, whereas total FNA compares only the initial and final outcomes. Thus, although the total FNA remains an important summary of overall harm, it does not reveal whether harms along one pathway are offset by benefits along another. Path-specific FNAs therefore complement the total FNA by providing a mechanism-level view of harm. A further discussion on the motivation and interpretation of the FNA parameters is provided in Appendix~\ref{app:fna-motivation-details}.

\begin{example}[Zero total FNA with non-zero path-specific FNA]
    Consider a population consisting of two equally sized subgroups. Assume $Y(0,M(0)) = 1, Y(1,M(0)) = 0, Y(1,M(1)) = 1$ for Group 1 and $Y(0,M(0)) = 0, Y(1,M(0)) = 1, Y(1,M(1)) = 0$ for Group 2. For Group 1, the \dir{direct} pathway is harmful since 
\(
Y(1,M(0)) < Y(0,M(0)),
\)
while the \ind{indirect} pathway is beneficial since 
\(
Y(1,M(1)) > Y(1,M(0)).
\)
For Group 2, the \ind{indirect} pathway is harmful while the \dir{direct} pathway is beneficial. Overall, $\tot{\fna_{\mathrm{tot}} = 0}$ which is not simply the sum $\dir{\fna_{\mathrm{dir}}}+ \ind{\fna_{\mathrm{ind}}} =1$.
Even though no individual is harmed under the total intervention, both the direct and indirect pathways exhibit non-zero harm. 
Hence, simply looking at \tot{$\fna_{\text{tot}}$} would ignore the different roles of \dir{direct} and \ind{indirect} pathways. 
\end{example}

\textbf{Objective.} The objective is to quantify the population fraction of \emph{individually harmed} units in terms of FNA. One challenge is that we never directly observe $Y(a,M(a'))$ where $a$ and $a'$ are counterfactual. Moreover, as the FNAs rely on the joint potential outcome distribution which is never observable even in RCTs, they are only partially identifiable. We give sharp bounds and then propose an orthogonal learner to estimate the bounds.

\section{Sharp covariate-assisted bounds for path-specific FNA}

In this section, we give the partial identification results underlying our path-specific harm analysis. We first focus on the identification of the conditional marginal distributions of the nested potential in Section~\ref{subsec:identifiability}. Then, we derive the Makarov bounds for path-specific FNAs in Section~\ref{subsec:bounds}.

\subsection{Identifiability of the nested marginals} \label{subsec:identifiability}

We begin by clarifying what can and cannot be identified from observed data. Each FNA is of the form $\Pr(U<V)$ for some pair $(U,V)$ with identified marginals but unknown joint dependence. Due to the fundamental problem of causal inference, the exact value of $\Pr(U<V)$ is not identifiable. However, while the joint distribution is unknown, the \emph{marginal} distributions of the nested potential outcomes can be identified under standard assumptions. These identified marginals will serve as the key inputs for constructing sharp bounds in the next subsection.

We make the standard identifiability assumptions for causal inference~\citep{robins1992identifiability, pearl2001direct}: (1)~\textbf{Consistency:} If $A=a$ then $M=M(a)$ and $Y=Y(a,M(a))$; (2)~\textbf{Sequential ignorability:} For all $a,a'\in\{0,1\}$, 
    $ \{Y(a',m),M(a)\} \perp A \mid X$ and  
    $Y(a',m) \perp M (a)\mid (A=a,X)$;
and (3)~\textbf{Overlap:} There exists $\epsilon>0$ such that $\epsilon \le \pi_1(X)\le 1-\epsilon$ a.s., and the mediator densities satisfy $g_a(m\mid X)>0$ whenever $g_{a'}(m\mid X)>0$ for the relevant $(a,a')$. Under these assumptions, the conditional CDFs of nested outcomes are identified by the mediation g-formula~\citep{imai2010mediationgeneral,robins1992identifiability}:
\begin{equation}
\eta_{a,a'}(t,X)
=\int \mu_a(t,m,X)\, g_{a'}(m\mid X)\, \diff m.
\label{eq:identified_marginal}
\end{equation}
Averaging  Eq.~\eqref{eq:identified_marginal} over $X$ yields the marginal CDFs: $F_{a,a'}(t) = \E[\eta_{a,a'}(t,X)] $.

\subsection{Derivation of sharp bounds for path-specific FNA}\label{subsec:bounds} 

The nested marginal potential outcome distributions can be identified, but the key challenge is that each FNA involves a comparison between two nested potential outcomes, for which the joint dependence is unknown. For two general random variables $U$ and $V$ with known marginals, the sharp bounds of $\Pr(U<V)$ arise from considering the most optimistic and pessimistic dependence structures between $U$ and $V$. Let $F_{U}$ and $F_{V}$ denote the CDFs of $U$ and $V$, respectively, and let us define $\Delta(t):=F_{U}(t)-F_{V}(t)$. The sharp Makarov bounds 
\citep{Makarov1982, Frank1987BestpossibleBF, heckman1997heterogeneity} characterize the tightest possible range of 
$\Pr(U<V)$ over all joint distributions consistent with these marginals: 
\begin{equation}
\underline{\theta} = \Big[\sup_{t\in\mathbb R}\Delta(t)\Big]_+,
\qquad
\overline{\theta} = 1 - \Big[\sup_{t\in\mathbb R}(-\Delta(t))\Big]_+,
\label{eq:makarov_sup_form}
\end{equation}
where $[z]_+ := \max\{z,0\}$.\footnote{For strict inequalities and discrete outcomes, left-limits may be used; we omit these technicalities for brevity.}. Plugging in the appropriately identified marginal CDFs into~\eqref{eq:makarov_sup_form} gives the sharp bounds for FNAs if only the unconditional marginals $F_{a,a'}(t)$ are available. However, in our setting, we observe additional covariates $X$ that provide more information about the data and thus can tighten the bounds \citep{FIRPO2019tightenbounds,Fan2010covariatessharpbounds}. Indeed, since $f\mapsto [\sup_t f(t)]_+$ is convex, Jensen's inequality implies
\begin{equation}
\mathbb E\left[
\Big[
\sup_t \Delta_X(t)
\Big]_+
\right]
\ge
\Big[
\sup_t \mathbb E\{\Delta_X(t)\}
\Big]_+
=
\Big[
\sup_t \{F_U(t)-F_V(t)\}
\Big]_+.
\end{equation}
Hence, the conditional lower bound is at least as large as the marginal lower bound, and the conditional upper bound is no larger than the marginal upper bound. So, instead of first averaging $\eta_{a,a'}(t,X)$ to obtain the marginal $F_{a,a'}(t)$ for Makarov bounds, we apply the Makarov bounds conditional on $X$ \emph{before} averaging over the distribution of $X$. This ordering preserves covariate-level information and yields tighter bounds, which leads to our first main result in the following.
\begin{tcolorbox}[
    colback=boxback, 
    colframe=boxborder, 
    arc=8pt,           
    outer arc=8pt, 
    boxrule=0.8pt,     
    left=10pt,         
    right=10pt,        
    top=5pt,          
    bottom=0.5pt        
]

\textbf{Sharp bounds for FNAs.} The sharp bounds for (in)direct FNA are given by
\begingroup
\fontsize{8}{11}\selectfont
\begin{equation}\label{eqn:fnabounds}
\begin{aligned}
\underline{\fna}_{\dir{\mathrm{dir}}}
&= \E\Big[\Big(\sup_{t\in\mathbb R}( F_{1,0}(t \mid X)-F_{0,0}(t \mid X))\Big)_+ \Big],
\overline{\fna}_{\dir{\mathrm{dir}}}
= 1 - \E\Big[\Big(\sup_{t\in\mathbb R}( F_{0,0}(t \mid X)-F_{1,0}(t \mid X))\Big)_+\Big] \\
\underline{\fna}_{\ind{\mathrm{ind}}}
&= \E\Big[\Big(\sup_{t\in\mathbb R}( F_{1,1}(t \mid X)-F_{1,0}(t \mid X))\Big)_+\Big], 
\overline{\fna}_{\ind{\mathrm{ind}}}
= 1 - \E\Big[\Big(\sup_{t\in\mathbb R}( F_{1,0}(t \mid X)-F_{1,1}(t \mid X))\Big)_+\Big]
\end{aligned}
\end{equation}
\endgroup
\end{tcolorbox}

\begin{remark}
    These bounds are \emph{sharp}, in the sense that every value in $[\underline{\theta},\overline{\theta}]$ is attainable by some joint distribution with the specified marginals~\citep{Makarov1982}. We provide a derivation in Appendix~\ref{app:sharp_bounds}.
\end{remark}

Since efficient inference for total FNA bounds is closely related to existing results, we focus below on the direct and indirect cases.

\section{Estimation and inference}\label{sec:estimation}

\textbf{Motivation: plug-in bias.}
A na{\"i}ve approach would estimate marginal CDFs and apply marginal Makarov bounds. However, this discards covariate information and yields unnecessarily wide bounds. A naturally better approach estimates the identified marginals $F_{a,a'}(t|X)$ by plug-in estimators of Eq.~\eqref{eq:identified_marginal}, and then plugs these into the supremum map from Eq.~\eqref{eqn:fnabounds}. However, the composition of nonparametric nuisance estimation and a non-smooth supremum functional can yield non-negligible finite-sample bias which is often called plug-in bias \citep{Farrell2015,ChernozhukovEtAl2018}.

To address plug-in bias, our proposed estimation involves two key improvements over the na{\"i}ve plug-in approach. First, at the identification level, we exploit covariate information to obtain tighter bounds by applying the Makarov construction conditionally on $X$ before averaging. This idea is related to classical results showing that conditioning sharpens bounds in partially identified problems~\citep{Fan2010covariatessharpbounds}. While conceptually straightforward, such conditioning introduces additional estimation challenges, as it requires learning conditional CDFs and evaluating a non-smooth supremum functional. This leads to our second improvement. By employing cross-fitting and standard double machine learning (DML), we develop a two-stage orthogonal learning procedure. The first stage constructs debiased pseudo-outcomes for which the conditional expectation equals the target CDF, and the second stage regress the debiased pseudo-outcomes on the covariates to learn these functions. Below, we first derive the efficient influence functions (EIFs) for problem (Section~\ref{sec:EIFs}) and then present an orthogonal estimator (Section~\ref{sec:orthogonal_estimator}).

\subsection{EIFs for Makarov bounds of direct and indirect FNA}
\label{sec:EIFs}

In the following, we state the target parameters and derive EIFs for the bounds. 

\emph{Notation.} For each pair $(a,a')$, we define the uncentred orthogonal score
\begin{equation}\label{eqn:orthogonalscore}
\begin{aligned}
\Gamma_{a,a'}^t(O)
&=
\frac{\mathbbm 1(A=a)}{\pi_a(X)}
\frac{g_{a'}(M\mid X)}{g_a(M\mid X)}
\big(Y_t-\mu_a(t,M,X)\big) \\
&\quad
+ \frac{\mathbbm 1(A=a')}{\pi_{a'}(X)}
\big(\mu_a(t,M,X)-\eta_{a,a'}(t,X)\big)
+ \eta_{a,a'}(t,X),
\end{aligned}
\end{equation}
where $Y_t := \mathbbm 1\{Y \le t\}$. By construction, we have $\mathbb E\!\left[\Gamma_{a,a'}^t(O)\mid X\right]=F_{a,a'}(t\mid X)$. That is, $\Gamma_{a,a'}^{\,t}(O)$ is a \emph{debiased pseudo-outcome}: a noisy function of the full observation $O$ whose conditional expectation given $X$ recovers the \emph{conditional} CDF $F_{a,a'}(t\mid X)=\eta_{a,a'}(t,X)$.

For notational convenience, let $(U,V)$ denote the relevant ordered pair of nested potential outcomes; for direct FNA, $(U,V)=\bigl(Y(1,M(0)),\,Y(0,M(0))\bigr)$ and for indirect FNA, $(U,V)=\bigl(Y(1,M(1)),\,Y(1,M(0))\bigr)$. Using the appropriate $(a,a')$ pairs, we  define the
\emph{pointwise CDF-difference score}
\begin{equation}
\label{eqn:deltascore}
\Gamma_\Delta^{\,t}(O)
\;:=\;
\Gamma_U^{\,t}(O) - \Gamma_V^{\,t}(O),
\end{equation}
where $\Gamma_U^{\,t}$ and $\Gamma_V^{\,t}$ are instances
of Eq.~\eqref{eqn:orthogonalscore}.
By linearity of conditional expectation,
\begin{equation}
\label{eq:delta_score_mean}
\mathbb E\!\bigl[\Gamma_\Delta^{\,t}(O)\mid X\bigr]
\;=\;
F_U(t\mid X) - F_V(t\mid X)
\;=\;
\Delta_X(t).
\end{equation}
Note that $\Gamma_\Delta^{\,t}(O)$ is a function of the full observation $O$,
not a conditional object; conditioning on $X$ is only invoked to identify
its mean~\eqref{eq:delta_score_mean}.

To streamline notation, let $\theta \in \{ \dir{\fna_{dir}}, \ind{\fna_{ind}}\}$ denote a generic FNA parameter. We use $\underline{\theta}$ and $\overline{\theta}$ to denote the lower and upper bounds. Then, the target bounds can be written as
\begin{equation}
\underline{\theta}
=
\E\left[\left(\sup_{t\in\mathbb R}\Delta_X(t)\right)_+\right],
\qquad
\overline{\theta} =1-\E\left[\left(\sup_{t\in\mathbb R}\{-\Delta_X(t)\}\right)_+\right].
\label{eq:conditional_makarov_bounds_general}
\end{equation}

For our EIF result, we assume the usual regularity for supremum functionals: the covariate-specific maximisers of $\Delta_X(t)$ and $-\Delta_X(t)$ are locally well separated, and the positive-part boundary is avoided, i.e., $\Pr\{\sup_t\Delta_X(t)=0\}=\Pr\{\sup_t-\Delta_X(t)=0\}=0$. We later refer to it as the margin condition. 
\begin{theorem}[EIFs for conditional direct/indirect FNA bounds]
\label{thm:eif_bounds}
Consider either 
$\Delta_{\mathrm{dir},X}(t)=F_{1,0}(t\mid X)-F_{0,0}(t\mid X)$
or
$\Delta_{\mathrm{ind},X}(t)=F_{1,1}(t\mid X)-F_{1,0}(t\mid X)$. Under the conditional margin condition, the conditional Makarov bounds in Eq.~\eqref{eq:conditional_makarov_bounds_general} are pathwise differentiable, with EIFs
\begin{align}
\phi_{\underline{\theta}}(O)
&=
\mathbbm 1\{\sup_{t}\Delta_X(t)>0\}\,
\Gamma_{\Delta}^{\,\underline{t}_X}(O)
-\underline{\theta},
\\
\phi_{\overline{\theta}}(O)
&=
1+
\mathbbm 1\{\sup_{t}(-\Delta_X(t))>0\}\,
\Gamma_{\Delta}^{\,\overline{t}_X}(O)
-\overline{\theta},
\end{align}
where $\underline{t}_X \in \arg\max_{t\in\mathbb R}\Delta_X(t)$ and $\overline{t}_X \in \arg\max_{t\in\mathbb R}\{-\Delta_X(t)\}$, and $\Gamma_{\Delta}^{\cdot}$ is defined in Eq.~\eqref{eqn:deltascore}.
\end{theorem}
\vspace{-0.3cm}
\begin{proof}
See Appendix~\ref{app:eif}.
\end{proof}
\vspace{-0.3cm}

Intuitively, the EIF evaluates the pointwise orthogonal score at the covariate-specific value of $t$ where the CDF difference is maximized. Crucially, the score $\Gamma_\Delta^{\,t}(O)$ already encodes
the relevant conditional CDFs through the nuisance functions
$(\pi_a, g_a, \mu_a, \eta_{a,a'})$;
no additional conditioning step is required.
The indicator $\mathbbm 1\{\sup_t\Delta_X(t)>0\}$ acts as a
\emph{covariate-specific switch}: it ensures the lower bound
contribution from a unit is zeroed out whenever the conditional
CDF difference has no positive part.

\vspace{-0.1cm}
\subsection{Two-stage orthogonal estimators for FNA bounds}
\label{sec:orthogonal_estimator}
\vspace{-0.1cm}

To estimate the EIFs from Theorem~\ref{thm:eif_bounds}, observe that the EIFs depend on three objects: (a) the nuisance functions $(\pi_a, g_a, \mu_a)$ entering~$\Gamma_{a,a'}^{\,t}$; (b) the conditional CDFs $F_{a,a'}(t\mid X)$ and the maximisers $\overline{\underline{t}}_X$; and (c) the data-adaptive maximisers themselves. The first component is standard: we can obtain cross-fitted estimates of the orthogonal score  $\widehat\Gamma_{a,a'}^{\,t}(O)$ by plugging nuisance estimates into Eq.~\eqref{eqn:orthogonalscore}. However, the main challenge lies in estimating the \emph{conditional CDFs} and the associated maximisers in a way that preserves the orthogonality required for valid inference. Next, we describe (1) how to estimate the conditional CDFs and (2) how to construct orthogonal estimators for the FNA bounds.

\textbf{(1) Estimate conditional CDFs as functions of $x$.}
$\widehat\Gamma_{a,a'}^t(O)$ serves as a debiased pseudo-outcome for $F_{a,a'}(t\mid X)$. Importantly, the pseudo-outcome $\widehat\Gamma_{a,a'}^t(O)$ itself is \emph{not} an estimate of $F_{a,a'}(t\mid X)$, but rather a noisy observation whose conditional 
expectation equals this quantity. However, for the covariate-assisted bounds, we require the conditional mapping  $t \mapsto F_{a,a'}(t\mid X)$ for each value of $X$.
We recover this function by regressing the pseudo-outcomes on $X$: 
\begin{equation}\label{eq:second_stage}
\widehat F_{a,a'}(t\mid x)\;:=\;\widehat{\mathbb E}\!\bigl[\widehat\Gamma_{a,a'}^{\,t}(O)\mid X=x\bigr],
\end{equation}

This second-stage regression yields a smooth estimate of the function $t \mapsto F_{a,a'}(t\mid X)$ while preserving the orthogonality property of the first-stage pseudo-outcomes. In principle, the regression can be implemented using a variety of flexible learners. In our experiments, we implement~\eqref{eq:second_stage} with a neural network trained under the CRPS loss~\citep{gneiting2007strictly}, which directly targets the full conditional CDF as a function of $t$ with monotonicity enforced as a hard architectural constraint.

\textbf{(2) Construct orthogonal estimators for bounds.}
Given estimates of the conditional CDFs, the FNA bounds are defined through pointwise differences. For a generic pair \((U,V)\), let $\widehat\Delta_X(t):=\widehat F_U(t\mid X)-\widehat F_V(t\mid X)$. We approximate the supremum over $t$ using a finite grid $\mathcal T_n$, and obtain data-adaptive maximisers:
\begin{equation*}
\hat {\underline{t}}(X)\in\arg\max_{t\in\mathcal T_n}\widehat\Delta_X(t),
\qquad
\hat {\overline{t}}(X)\in\arg\max_{t\in\mathcal T_n}\bigl\{-\widehat\Delta_X(t)\bigr\}.
\end{equation*}
These maximisers identify the thresholds at which the lower and upper bounds are attained for each covariate value.

We now state our estimators for the FNA bounds. Let $\widehat\Gamma_\Delta^{\,t}(O):=\widehat\Gamma_U^{\,t}(O)-\widehat\Gamma_V^{\,t}(O)$, which is the empirical analogue of the corresponding quantities defined in the previous subsection. Plugging the data-adaptive maximisers into the EIF yields the orthogonal (one-step) estimators.

\begin{tcolorbox}[
    colback=boxback, 
    colframe=boxborder, 
    arc=8pt,           
    outer arc=8pt, 
    boxrule=0.8pt,     
    left=10pt,         
    right=10pt,        
    top=5pt,          
    bottom=5pt        
]
\textbf{Orthogonal estimators for FNA bounds.} For $\theta \in \{\fna_{dir}, \fna_{ind}\}$, the orthogonal (one-step) estimators are:
\begin{align}
\hspace{-0.3cm}
\widehat{\underline{\theta}}
&=
\mathbb P_n\!\Bigl[
\mathbbm 1\bigl\{\widehat\Delta_X(\hat {\underline{t}}_X)>0\bigr\}\,
\widehat\Gamma_\Delta^{\,\hat {\underline{t}}_X}(O)
\Bigr],
\quad
\widehat{\overline{\theta}}
&=
\mathbb P_n\!\Bigl[
1+
\mathbbm 1\bigl\{-\widehat\Delta_X(\hat {\overline{t}}_X)>0\bigr\}\,
\widehat\Gamma_\Delta^{\,\hat {\overline{t}}_X}(O)
\Bigr].
\end{align}
\end{tcolorbox}

Intuitively, our above estimators proceed by (i)~ locating the threshold where harm is worst, and (ii)~evaluating this worst-case discrepancy using a debiased score. This construction allows us to control for bias arising from both nuisance estimation \emph{and} the supremum operator.

Given the estimates of the bounds, the estimated influence functions are
\begin{align}
\widehat\phi_{\underline{\theta}}(O)
&=
\mathbbm 1\bigl\{\widehat\Delta_X(\hat {\underline{t}}_X)>0\bigr\}\,
\widehat\Gamma_\Delta^{\,\hat {\underline{t}}_X}(O)
-\widehat{\underline{\theta}},
\label{eq:eif_est_lower}
\\
\widehat\phi_{{\overline{\theta}}}(O)
&=
1
+\mathbbm 1\bigl\{-\widehat\Delta_X(\hat {\overline{t}}_X)>0\bigr\}\,
\widehat\Gamma_\Delta^{\,\hat {\overline{t}}_X}(O)
-\widehat{\overline{\theta}}.
\label{eq:eif_est_upper}
\end{align}

\textbf{Valid confidence intervals (CIs).} Next, we describe how to obtain valid CIs and conservative sets for the estimated bounds. 

\begin{assumption}[Regularity for one-step inference]
\label{ass:main_dml_regular}
The nuisance estimators are cross-fitted and satisfy the usual overlap, consistency,
and second-order product-rate conditions for the orthogonal scores
$\Gamma_\Delta^t$. In addition, the threshold grid and the estimated maximisers are
stable in the sense that replacing the population maximisers by the data-adaptive
maximisers changes the target value and the corresponding empirical score only by
$o_p(n^{-1/2})$. The influence functions have finite, nonzero variance, and the
estimated influence functions are $L_2(P)$ consistent.
\end{assumption}

The regularity condition above is high-level. In
Appendix~\ref{app:asymptotic_normality} we state primitive sufficient conditions.

\begin{theorem}[Efficient inference for conditional direct/indirect FNA bounds]
\label{thm:dml_clt}
Let $\widehat{\underline\theta}$ and $\widehat{\overline\theta}$ be the cross-fitted
one-step estimators of the lower and upper conditional FNA bounds
$\underline\theta$ and $\overline\theta$. Suppose Assumption~\ref{ass:main_dml_regular}
holds. Then, $
\sqrt{n}\big(\widehat{\overline{\underline{\theta}}}-{\overline{\underline{\theta}}}\big)
=
\frac{1}{\sqrt{n}}\sum_{i=1}^n
\phi_{\overline{\underline{\theta}}}(O_i)
+
o_p(1)$.
Consequently,
$\frac{\widehat{\overline{\underline{\theta}}}-{\overline{\underline{\theta}}}}
{\widehat{\mathrm{se}}({\overline{\underline{\theta}}})}
\ \overset{d}{\longrightarrow}\ 
\mathcal N(0,1)$, 
with the estimated standard errors $
\widehat{\mathrm{se}}(\widehat{\overline{\underline{\theta}}})
:=
\sqrt{\mathbb V_n(\widehat\phi_{\overline{\underline{\theta}}})/n}$.

Furthermore, \(\widehat{\overline{\underline{\theta}}}\) is asymptotically efficient.
A \((1-\alpha)\) Wald interval for each bound is
\begin{equation}
\overline{\underline{\mathrm{CI}}}_{1-\alpha}
=
\left[
\widehat{\overline{\underline{\theta}}}
\pm
z_{1-\alpha/2}\,
\widehat{\mathrm{se}}(\widehat{\overline{\underline{\theta}}})
\right]\cap[0,1].
\end{equation}
A conservative confidence set for the partially identified FNA parameter 
\(\theta\in[\underline{\theta},\overline{\theta}]\) is
\begin{equation}
\left[
\max\{0,\ \widehat{\underline{\theta}}-z_{1-\alpha/2}\widehat{\mathrm{se}}(\widehat{\underline{\theta}})\},
\ 
\min\{1,\ \widehat{\overline{\theta}}+z_{1-\alpha/2}\widehat{\mathrm{se}}(\widehat{\overline{\theta}})\}
\right].
\end{equation}
\end{theorem}
This result provides principled uncertainty quantification for harm. In particular, we obtain conservative sets for the fraction of individuals harmed by an intervention which provides a measure for worst-case FNA.

\section{Experiments}
\label{sec:experiments}

The main objective of our numerical experiments is to empirically validate our theoretical finding. Our experiments are inspired by standard practice in the partial identification literature~\citep{schroder2024causal,Kallus2022FNA_sharp_bounds}: \circledgreen{i} to show that our proposed orthogonal estimator has better finite-sample performance relative to plug-in estimators, and \circledgreen{ii} to demonstrate that our orthogonal estimator yields valid confidence intervals for the FNA bounds.

\textbf{Data-generating process (DGP).}
We simulate data with sample size $n=5000$ from a nonlinear causal mediation model with a continuous covariate $X$, binary treatment $A$, binary mediator $M$, and continuous outcome $Y$. Treatment and mediator assignment follow logistic models, while the outcome model has heterogeneous direct and mediator effects that vary smoothly with $X$. Full details of the DGP are given in Appendix~\ref{app:dgp}. 

\textbf{Estimators.}
We benchmark our proposed orthogonal estimator (\textbf{orthogonal)} against the plug-in estimator (\textbf{plug-in}). Both estimators target the covariate-assisted Makarov bounds in Eq.~\eqref{eqn:fnabounds}. The plug-in estimator directly inserts estimated conditional CDFs into the bound functionals. In contrast, the orthogonal estimator first constructs cross-fitted debiased pseudo-outcomes for the nested CDFs and then learns the conditional maps \(F_{a,a'}(t\mid X)\) by second-stage regression. The true conditional CDFs are available in closed form and are used only to compute oracle benchmark bounds. To ensure a \emph{fair} comparison, both estimators use the same nuisance components, model classes, and tuning, so that any performance differences are solely due to the estimation strategy.
 The full implementation details are provided in Appendix~\ref{app:experimental_details}.

\textbf{Performance metrics.} We report several performance metrics to evaluate the estimators. For each pathway, (i)~\textbf{FNA} denotes the true fraction negatively affected, and \textbf{Oracle} denotes the oracle covariate-assisted bounds. For each estimator, (ii)~\textbf{Mean est.} reports the average estimated lower and upper bounds across simulations. (iii)~\textbf{Bias} is defined as the difference between the estimated and oracle bounds, reported separately for the lower and upper bound. (iv)~\textbf{Coverage} measures the empirical coverage of the oracle bounds by the corresponding conservative sets. Finally, (v)~\textbf{Width} denotes the average length of the estimated bounds. Ideally, a width comparable to the oracle width is desirable.

\textbf{Results.}
Table~\ref{tab:oracle_effects_fna} reports oracle average effects and true path-specific FNAs under the DGP. 
\begin{wraptable}[9]{r}{0.45\textwidth} 
\centering
\vspace{-0.16cm}
\caption{\textbf{Oracle average effects and true path-specific FNA.}}
\label{tab:oracle_effects_fna}
\footnotesize
\begin{tabular}{lll}
\toprule
Pathway & Average effect & True FNA \\
\midrule
Direct & \(\mathrm{NDE}=1.154\) & \(\fna_{\mathrm{dir}}=0.410\) \\
Indirect & \(\mathrm{NIE}=0.067\) & \(\fna_{\mathrm{ind}}=0.223\) \\
Total & \(\mathrm{ATE}=1.220\) & \(\fna_{\mathrm{tot}}=0.220\) \\
\bottomrule
\end{tabular}
\end{wraptable}
We make the following observations. Positive average effects can exist, even with nonzero path-specific FNA. For example, although the average direct effect is positive, $\mathrm{NDE}=1.154>0$, the direct FNA is also substantial. Similarly, the total average effect is positive, $\mathrm{ATE}=1.22>0$, while the total FNA remains nonzero. This illustrates the central motivation for path-specific FNA: \emph{average pathway effects alone do not quantify how often individuals are made worse off.}

Table~\ref{tab:estimation_results} compares the plug-in and orthogonal estimators for the lower and upper bounds of direct, indirect, and total FNA. Results are based on $300$ runs with different seeds. We make the following observations. (1)~The orthogonal estimator has smaller bias than the plug-in estimator, especially for indirect and total FNA bounds. This confirms the effectiveness of our method. The improvement is particularly pronounced in the moderate-sample regime considered here, where plug-in estimation suffers from nuisance-estimation bias compounded by the non-smooth supremum operation in the Makarov functional. (2)~The orthogonal estimator achieves perfect empirical coverage for the estimated bounds, indicating reliable uncertainty quantification. (3)~The orthogonal estimator also produces substantially narrower intervals while maintaining coverage, reflecting tighter and more informative bounds. $\Rightarrow$ \textbf{Takeaway:} \emph{The results directly support our theoretical claims: our orthogonal estimator achieves \circledgreen{i} improved finite-sample performance over plug-in estimators, and \circledgreen{ii} valid confidence intervals for the FNA bounds.}

\begin{table}[ht]
\centering
\caption{\textbf{Finite-sample ($n=5000$) performance for estimating covariate-assisted FNA bounds.} Results: mean over 300 simulations. Bold indicates the better-performing estimator within each pathway.}
\label{tab:estimation_results}
\footnotesize
\setlength{\tabcolsep}{2pt}
\begin{tabular}{clccccccl}
\toprule
\textbf{$d$}
& \textbf{Pathway}
& \textbf{FNA}
& \textbf{Oracle}
& \textbf{Estimator}
& \textbf{Mean est.}
& \textbf{Bias}
& \(\mathbf{Coverage}\)
& \textbf{Width} \\
\midrule
\multirow{6}{*}{\rotatebox[origin=c]{90}{$d=3$}}
& \multirow{2}{*}{Direct}
& \multirow{2}{*}{\(0.410\)}
& \multirow{2}{*}{\([0.220,\,0.450]\)}
& Plug-in
& \([0.124,\,0.492]\)
& \((-0.096,\,+0.042)\)
& \({1.000}\)
& \(0.368\) \\
&
&
&
& \textbf{Orthogonal}
& \(\mathbf{[0.156,\,0.463]}\)
& \((\mathbf{-0.064},\,\mathbf{+0.013})\)
& \(\mathbf{1.000}\)
& \(\mathbf{0.307}\) \\
\addlinespace[1pt]
\cmidrule(l){2-9}
& \multirow{2}{*}{Indirect}
& \multirow{2}{*}{\(0.223\)}
& \multirow{2}{*}{\([0.177,\,0.797]\)}
& Plug-in
& \([0.037,\,0.954]\)
& \((-0.140,\,+0.157)\)
& \(1.000\)
& \(0.917\) \\
&  
& 
&
& \textbf{Orthogonal}
& \(\mathbf{[0.128,\, 0.850]}\)
& \((\mathbf{-0.049},\,\mathbf{+0.053})\)
& \(\mathbf{1.000}\)
& \(\mathbf{ 0.722}\) \\
\addlinespace[1pt]
\cmidrule(l){2-9}
& \multirow{2}{*}{Total}
& \multirow{2}{*}{\(0.220\)}
& \multirow{2}{*}{\([ 0.082,\,0.333]\)}
& Plug-in
& \([0.104,\,0.475]\)
& \((+0.022,\,+0.142)\)
& \(0.070\)
& \(0.371\) \\
&  
&
&
& \textbf{Orthogonal}
& \(\mathbf{[0.069,\,0.361]}\)
& \((\mathbf{-0.013},\,\mathbf{+0.028})\)
& \(\mathbf{1.000}\)
& \(\mathbf{0.292}\) \\
\midrule
\midrule
\multirow{6}{*}{\rotatebox[origin=c]{90}{$d=5$}}
& \multirow{2}{*}{Direct}
& \multirow{2}{*}{\(0.410\)}
& \multirow{2}{*}{\([0.221,\,0.450]\)}
& Plug-in
& \([0.116,\,0.495]\)
& \((-0.105,\,+0.045)\)
& \(1.000\)
& \(0.379\) \\
&
&
& 
& \textbf{Orthogonal}
& \(\mathbf{[0.149,\,0.462]}\)
& \(\mathbf{(-0.072,\,+0.012)}\)
& \(\mathbf{1.000}\)
& \(\mathbf{0.348}\) \\
\addlinespace[1pt]
\cmidrule(l){2-9}
& \multirow{2}{*}{Indirect}
& \multirow{2}{*}{\(0.223\)}
& \multirow{2}{*}{\([0.175,\,0.798]\)}
& Plug-in
& \([0.035,\,0.955]\)
& \((-0.140,\,+0.157)\)
& \(1.000\)
& \(0.920\) \\
&
&
&
& \textbf{Orthogonal}
& \(\mathbf{[0.113,\,0.864]}\)
& \(\mathbf{(-0.062,\,+0.066)}\)
& \(\mathbf{1.000}\)
& \(\mathbf{0.751}\) \\
\addlinespace[1pt]
\cmidrule(l){2-9}
& \multirow{2}{*}{Total}
& \multirow{2}{*}{\(0.220\)}
& \multirow{2}{*}{\([0.084,\,0.334]\)}
& Plug-in
& \([0.097,\,0.483]\)
& \((-0.013,\,+0.149)\)
& \(0.897\)
& \(0.386\) \\
&
&
&
& \textbf{Orthogonal}
& \(\mathbf{[0.066,\,0.375]}\)
& \(\mathbf{(-0.018,\,+0.041)}\)
& \(\mathbf{1.000}\)
& \(\mathbf{0.309}\) \\
\midrule
\midrule
\multirow{6}{*}{\rotatebox[origin=c]{90}{$d=10$}}
& \multirow{2}{*}{Direct}
& \multirow{2}{*}{\(0.410\)}
& \multirow{2}{*}{\([0.193,\,0.451]\)}
& Plug-in
& \([ 0.108,\,0.498]\)
& \((-0.085,\,+0.047)\)
& \(1.000\)
& \(0.390\) \\
&  
&
&
& \textbf{Orthogonal}
& \(\mathbf{[0.116,\,0.462]}\)
& \(\mathbf{(-0.077,\,+0.011)}\)
& \(\mathbf{1.000}\)
& \(\mathbf{0.346}\) \\
\addlinespace[1pt]
\cmidrule(l){2-9}
& \multirow{2}{*}{Indirect}
& \multirow{2}{*}{\(0.223\)}
& \multirow{2}{*}{\([0.175,\,0.805]\)}
& Plug-in
& \([0.030,\,0.961]\)
& \((-0.145,\,+0.156)\)
& \(1.000\)
& \(0.931\) \\
& 
& 
&
& \textbf{Orthogonal}
& \(\mathbf{[0.089,\, 0.883]}\)
& \(\mathbf{(-0.086,\,+0.078)}\)
& \(\mathbf{1.000}\)
& \(\mathbf{0.794}\) \\
\addlinespace[1pt]
\cmidrule(l){2-9}
& \multirow{2}{*}{Total}
& \multirow{2}{*}{\(0.220\)}
& \multirow{2}{*}{\([0.086,\,0.340]\)}
& Plug-in
& \([0.096,\,0.491]\)
& \((+0.010,\,+0.151)\)
& \(0.250\)
& \(0.395\) \\
&  
&
& 
& \textbf{Orthogonal}
& \(\mathbf{[0.054,\,0.395]}\)
& \(\mathbf{(-0.032,\,+0.055)}\)
& \(\mathbf{1.000}\)
& \(\mathbf{0.341}\) \\
\bottomrule
\multicolumn{9}{p{\textwidth}}{\footnotesize $\mathbf{d}$ is the covariate dimension. \textbf{Oracle} denotes the oracle covariate-assisted bounds; \textbf{Mean est.} denotes the average estimated bounds; \textbf{FNA} is the true FNA value. \textbf{Bias} is reported as lower/upper bound bias of the estimated bounds against the oracle bounds; \textbf{Coverage} is the fraction of conservative sets covering the oracle bounds; \textbf{Width} is the mean estimated interval width. For all cases, the conservative sets covered the true FNA in 100\% of simulations.}
\end{tabular}
\end{table}

\paragraph{Real-world case study.}
We illustrate our method on the publicly available Moving to Opportunity (MTO) experiment \citep{ludwig2013mto}, a randomised housing-voucher intervention designed to help low-income families move to lower-poverty neighbourhoods.

Prior analyses have shown that relocation improves adult mental health and subjective well-being, with no adverse average effect on physical health or economic self-sufficiency~\citep{ludwig2012science,ludwig2013mto}. However, these are average effects that can conceal path-specific harm. In particular, among adolescent boys, voucher receipt \emph{increased} long-term risk of mood and externalizing disorders, with the majority of this adverse effect operating indirectly through changes in the neighbourhood and school environment~\citep{rudolph2021MOT}. Whether comparable harm lies beneath the favourable adult averages has not been examined.
We define the treatment $A$ as the offer of a relocation voucher (the experimental low-poverty and Section~8 arms, pooled) versus the control arm; the mediator $M$ as duration-weighted neighbourhood poverty exposure; and the outcome $Y$ as the standardized K6 psychological-distress index, sign-reversed so that larger values indicate better adult mental health. We adjust for baseline demographic, socioeconomic, household, housing, and study-site covariates $X$. Under this setting, indirect harm captures adults whose outcomes are worsened through the voucher-induced change in neighbourhood poverty, while direct harm captures worsening not operating through this measured mediator.

Table~\ref{tab:mto_plugin_vs_dr} reports the estimated direct, indirect, and total FNA bounds. We find: 
\begin{wraptable}{r}{0.68\textwidth}
\centering
\caption{\textbf{Estimated FNA bounds on the MTO data.}
Standard errors are shown in parentheses. The conservative set is obtained
by expanding the estimated lower and upper bounds by $1.96$ standard errors.
Plug-in standard errors ignore nuisance estimation and are reported for
reference only.}
\label{tab:mto_plugin_vs_dr}
\footnotesize
\setlength{\tabcolsep}{1pt}
\begin{tabular}{@{}llcc@{}}
\toprule
\textbf{Pathway} & \textbf{Estimator} & \textbf{Est. bounds} & \textbf{Conservative set} \\
\midrule
\multirow{2}{*}{Direct}
& Plug-in
& $[0.122\ (0.003),\,0.615\ (0.005)]$
& $[0.116,\,0.625]$ \\
& \textbf{Orth.}
& $\mathbf{[0.319\ (0.023),\,0.459\ (0.043)]}$
& $\mathbf{[0.274,\,0.542]}$ \\
\midrule
\multirow{2}{*}{Indirect}
& Plug-in
& $[0.036\ (0.000),\,0.996\ (0.000)]$
& $[0.035,\,0.997]$ \\
& \textbf{Orth.}
& $\mathbf{[0.042\ (0.039),\,0.935\ (0.018)]}$
& $\mathbf{[0.000,\,0.969]}$ \\
\midrule
\multirow{2}{*}{Total}
& Plug-in
& $[0.133\ (0.004),\,0.639\ (0.005)]$
& $[0.126,\,0.650]$ \\
& \textbf{Orth.}
& $\mathbf{[0.337\ (0.014),\,0.482\ (0.018)]}$
& $\mathbf{[0.310,\,0.517]}$ \\
\bottomrule
\end{tabular}
\vspace{-0.5cm}
\end{wraptable}(1) The strictly positive lower bound on direct FNA implies that at least 27\% participants are harmed through channels unrelated to neighbourhood-poverty exposure. 
This is a substantial minority that is otherwise masked by the program's well-documented positive average effect on adult mental health~\citep{ludwig2013mto}. 
(2) The indirect bounds are wide: the data provide little evidence establishing indirect harm, but also cannot exclude substantial harm through the neighbourhood-poverty pathway. This can be expected given the experimental setting. (3)~The bounds from the plug-in learner are wider than those from our learner. \emph{This is in line with theory and hints that our learner is beneficial.}

\section{Discussion}

$\bullet$\,\textbf{Conclusion:} To the best of our knowledge, this is the first work to study path-specific harm through direct and indirect FNA and to provide a corresponding partial identification framework. 
$\bullet$\,\textbf{Practical implications:} Our framework can make policy evaluation more transparent by showing \emph{where} harm arises along causal pathways. For example, a decision maker could compare policies using both expected benefit and a bound on indirect FNA, such as maximizing CATE subject to an acceptable harm constraint, in the spirit of harm-aware policy learning~\citep{BenMichael2024assymmetric}. In a drug setting, this corresponds to improving outcomes while limiting mediator-related side effects; more generally, we provide a mechanism-aware complement to average effect policy learning.
$\bullet$\,\textbf{Ethical discussion:} Our goal is not to claim a universally correct definition of harm, but to add to our theoretical understanding that complements existing criteria. In practice, the framework can support more transparent sensitivity analysis of where harm may arise along causal pathways. At the same time, our analysis relies on standard causal assumptions (e.g., no unmeasured confounding and correct model specification), which may be violated in real-world settings. Moreover, notions of harm are inherently normative, and the FNA may not fully align with the preferences or objectives of decision makers in all applications. We therefore recommend cautious use, that prioritises a safe, ethical, and reliable approach.

\printbibliography

\newpage

\appendix

\section{Notation}
\label{app:notation}
We summarise the notations used throughout the paper below in Table~\ref{tab:notation}. 
\begin{longtable}{p{0.22\textwidth} p{0.70\textwidth}}
\caption{Summary of notations and abbreviations used throughout the paper.}
\label{tab:notation}\\
\toprule
\textbf{Notation} & \textbf{Meaning} \\
\midrule
\endfirsthead

\toprule
\textbf{Notation} & \textbf{Meaning} \\
\midrule
\endhead

\midrule
\multicolumn{2}{r}{\emph{Continued on next page}}\\
\endfoot

\multicolumn{2}{l}{\textbf{Abbreviations}}\\[2pt]

FNA 
& Fraction negatively affected. \\

CDF 
& Cumulative distribution function. \\

EIF 
& Efficient influence function. \\

DML 
& Double/debiased machine learning. \\

RCT 
& randomised controlled trial. \\

CATE 
& Conditional average treatment effect. \\

CRPS 
& Continuous ranked probability score. \\
\multicolumn{2}{l}{\textbf{Observed data and variables}}\\[2pt]
$O=(X,A,M,Y)$ 
& Observed data unit. $X$: baseline covariates, $A$: treatment, $M$: mediator, $Y$: outcome \\

$t\in\mathbb R$ 
& Outcome threshold used to define distribution functions. \\

$\mathcal T_n$ 
& Finite threshold grid used to approximate suprema over $t$. \\

$Y_t=\mathbbm 1\{Y\le t\}$ 
& Binary threshold outcome. \\

\addlinespace[4pt]
\multicolumn{2}{l}{\textbf{Potential outcomes and path-specific comparisons}}\\[2pt]

$M(a)$ 
& Potential mediator value under treatment level $a$. \\

$Y(a)$ 
& Potential outcome under treatment level $a$. \\

$Y(a,m)$ 
& Potential outcome under treatment level $a$ and mediator value set to $m$. \\

$Y(a,M(a'))$ 
& Nested potential outcome under treatment $a$ with mediator drawn from its value under treatment $a'$. \\

$Y(0,M(0))$ 
& Baseline nested potential outcome. \\

$Y(1,M(0))$ 
& Potential outcome under treatment while fixing the mediator to its untreated value. \\

$Y(1,M(1))$ 
& Potential outcome under treatment with mediator set to its treated value. \\

$(U,V)$ 
& Generic pair of nested potential outcomes compared in an FNA parameter. \\

\addlinespace[4pt]
\multicolumn{2}{l}{\textbf{Fraction negatively affected}}\\[2pt]

$\fna$ 
& Fraction negatively affected: probability that one potential outcome is smaller than another. \\

$\fna_{\mathrm{dir}}$ 
& Direct FNA,
$\Pr\{Y(1,M(0))<Y(0,M(0))\}$. \\

$\fna_{\mathrm{ind}}$ 
& Indirect FNA,
$\Pr\{Y(1,M(1))<Y(1,M(0))\}$. \\

$\fna_{\mathrm{tot}}$ 
& Total FNA,
$\Pr\{Y(1,M(1))<Y(0,M(0))\}$. \\

$\theta$ 
& Generic FNA parameter, typically one of
$\fna_{\mathrm{dir}}$ or $\fna_{\mathrm{ind}}$. \\

$\underline{\theta},\overline{\theta}$ 
& Lower and upper sharp bounds for a generic FNA parameter $\theta$. \\

$\underline{\fna}_{\mathrm{dir}},
\overline{\fna}_{\mathrm{dir}}$ 
& Lower and upper sharp bounds for the direct FNA. \\

$\underline{\fna}_{\mathrm{ind}},
\overline{\fna}_{\mathrm{ind}}$ 
& Lower and upper sharp bounds for the indirect FNA. \\

$[z]_+=\max\{z,0\}$ 
& Positive part of a scalar $z$. \\

\addlinespace[4pt]
\multicolumn{2}{l}{\textbf{Nuisance functions and identified CDFs}}\\[2pt]

$\pi_a(x)$
& Treatment propensity score, $\Pr(A=a\mid X=x)$  \\

$g_a(m\mid x)$ 
& Conditional mediator density or mass function under treatment $a$, $f_{M\mid A=a,X}(m\mid x)$ \\

$\mu_a(t,m,x)$ 
& Conditional CDF of the observed outcome at threshold $t$, given treatment $a$, mediator $m$, and covariates $x$, $\mathbb E[Y_t\mid A=a,M=m,X=x]$ \\
 
$\eta_{a,a'}(t,x)$ 
& Mediator-averaged conditional CDF functional,
$\int \mu_a(t,m,x)g_{a'}(m\mid x)\,\diff m$. \\

$F_{a,a'}(t\mid X=x)$ 
& Conditional CDF of the nested potential outcome $Y(a,M(a'))$ given $X=x$; equal to $\eta_{a,a'}(t,x)$ under the identifying assumptions. \\

$F_{a,a'}(t)$ 
& Marginal CDF of $Y(a,M(a'))$,
$\Pr\{Y(a,M(a'))\le t\}=\mathbb E[\eta_{a,a'}(t,X)]$. \\

$F_U(t),F_V(t)$ 
& Marginal CDFs of generic random variables $U$ and $V$. \\

$F_U(t\mid X),F_V(t\mid X)$ 
& Conditional CDFs of the generic nested potential outcomes $U$ and $V$. \\

$\Delta(t)$&$F_U(t)-F_V(t)$ 
Marginal CDF difference. \\

$\Delta_X(t)$ 
& $F_U(t\mid X)-F_V(t\mid X)$, conditional CDF difference. \\

$\Delta_{\mathrm{dir},X}(t)$ 
& Direct-path CDF difference,
$F_{1,0}(t\mid X)-F_{0,0}(t\mid X)$. \\

$\Delta_{\mathrm{ind},X}(t)$ 
& Indirect-path CDF difference,
$F_{1,1}(t\mid X)-F_{1,0}(t\mid X)$. \\

\addlinespace[4pt]
\multicolumn{2}{l}{\textbf{Makarov bounds and maximisers}}\\[2pt]

$\sup_{t\in\mathbb R}\Delta_X(t)$ 
& Covariate-specific worst-case positive CDF discrepancy. \\

$\sup_{t\in\mathbb R}\{-\Delta_X(t)\}$ 
& Covariate-specific worst-case negative CDF discrepancy. \\

$\underline t_X$ 
& Covariate-specific maximiser of $\Delta_X(t)$ for the lower bound. \\

$\overline t_X$ 
& Covariate-specific maximiser of $-\Delta_X(t)$ for the upper bound. \\
$\hat {\underline{t}}(X),\hat {\overline{t}}(X)$ 
& Estimated maximisers over the finite grid $\mathcal T_n$. \\

\addlinespace[4pt]
\multicolumn{2}{l}{\textbf{Orthogonal scores and influence functions}}\\[2pt]

$\Gamma_{a,a'}^t(O)$ 
& Uncentered orthogonal score for $F_{a,a'}(t\mid X)$; satisfies
$\mathbb E[\Gamma_{a,a'}^t(O)\mid X]=F_{a,a'}(t\mid X)$. \\
$\widehat\Gamma_{a,a'}^t(O)$ 
& Estimated uncentred orthogonal score.\\

$\Gamma_U^t(O),\Gamma_V^t(O)$ 
& Orthogonal scores corresponding to general nested potential outcomes $U$ and $V$. \\

$\Gamma_\Delta^t(O)$
& Orthogonal score for the CDF difference $\Delta_X(t)$, $\Gamma_U^t(O)-\Gamma_V^t(O)$ . \\

$\phi_{\underline{\theta}}(O)$ 
& Efficient influence function for the lower bound $\underline{\theta}$. \\

$\phi_{\overline{\theta}}(O)$ 
& Efficient influence function for the upper bound $\overline{\theta}$. \\

$\widehat\Gamma_\Delta^t(O)$ 
& Estimated orthogonal score for the CDF difference. \\

$\widehat\phi_{\underline{\theta}}(O),
\widehat\phi_{\overline{\theta}}(O)$ 
& Estimated influence functions for the lower and upper bound estimators. \\

\addlinespace[4pt]
\multicolumn{2}{l}{\textbf{Estimators and inference}}\\[2pt]

$\widehat F_{a,a'}(t\mid x)$ 
& Second-stage estimate of the conditional nested CDF. \\

$\widehat\Delta_X(t)$ 
& Estimated conditional CDF difference. \\

$\mathbb P_n$ 
& Empirical average, $\mathbb P_n f = n^{-1}\sum_{i=1}^n f(O_i)$. \\

$\mathbb V_n$ 
& Empirical variance. \\

$\widehat{\underline{\theta}},
\widehat{\overline{\theta}}$ 
& One-step estimators of the lower and upper FNA bounds. \\

$\widehat{\mathrm{se}}(\widehat{\theta})$ 
& Estimated standard error of an estimator $\widehat{\theta}$. \\

$z_{1-\alpha/2}$ 
& $(1-\alpha/2)$ quantile of the standard normal distribution. \\

$\overline{\underline{\mathrm{CI}}}_{1-\alpha}$ 
& Wald confidence interval for a generic lower or upper bound. \\

\addlinespace[4pt]
\multicolumn{2}{l}{\textbf{Assumptions and probability notation}}\\[2pt]

$\mathbb P,\Pr$ 
& Probability under the data-generating distribution. \\

$\mathbb E$ 
& Expectation under the data-generating distribution. \\

$\mathbbm 1\{\cdot\}$ 
& Indicator function. \\

$\perp\!\!\!\perp$ 
& Statistical independence. \\

$a.s.$ 
& Almost surely. \\

$\overset{d}{\longrightarrow}$ 
& Convergence in distribution. 

\end{longtable}

\newpage

\section{Additional discussion of FNA interpretation and path-specific harm}
\label{app:fna-motivation-details}

\subsection{Differences between FNA and (C)ATE}
\label{app:fna-vs-average}
The motivation for FNA is not that it is universally preferable to average-effect criteria. Rather, it targets a different decision question. The average treatment effect,
\begin{equation*}
\mathrm{ATE}=\mathbb{E}[Y(1)-Y(0)] = \mathbb{E}[Y(1)]-\mathbb{E}[Y(0)],
\end{equation*}
depends only on the two marginal potential-outcome distributions. By contrast,
\begin{equation*}
\fna = \Pr\{Y(1)<Y(0)\}
=\mathbb{E}\left[\mathbbm{1}\{Y(1)<Y(0)\}\right]
\end{equation*}
counts the fraction of units for whom the intervention makes the outcome worse. It therefore depends on the \emph{joint} distribution of the two potential outcomes. The two quantities answer different questions: the ATE asks whether the intervention changes average welfare, while FNA asks how often the intervention worsens an individual's outcome relative to that individual's alternative outcome. Consequently, a positive ATE can coexist with a positive FNA. 
A decision maker concerned with the number of people harmed cannot in general replace the harm event with the sign of an average effect.

A natural alternative is to condition on observed characteristics $Z=z$ and inspect the $Z$-specific effect,
\begin{equation*}
\tau(z)=\mathbb{E}[Y(1)-Y(0)\mid Z=z].
\end{equation*}
This is useful for studying treatment-effect heterogeneity and can reveal subgroups with negative average effects. However, it remains a contrast of two \emph{marginal} means within each stratum. Repeating the comparison for finer strata does not identify the within-unit event $\{Y(1)<Y(0)\}$ inside the stratum. Two units with the same observed $Z=z$ can have different unobserved pairs $(Y_i(0),Y_i(1))$ and hence different harm status. Thus, conditioning can help tighten FNA bounds, but conditioning on an estimated ``harmed subgroup'' is not an alternative identification strategy: the harmed subgroup is itself defined by an unobserved cross-world event.

This distinction is also why the present framework should not be interpreted as a claim that FNA is the universally best harm notion. Average effects, distributional criteria, subgroup effects, and FNA encode different normative objectives. Our contribution is to provide a way to quantify a particular and practically relevant objective---the population fraction of individually harmed units---and, in the mediation setting, to localize that harm along causal pathways.

\subsection{Nonadditivity of direct and indirect FNA}
\label{app:nonadditivity}
The nonadditivity is easiest to see by writing the nested potential outcomes as a two-step trajectory,
\begin{equation*}
Y(0,M(0)) \longrightarrow Y(1,M(0)) \longrightarrow Y(1,M(1)).
\end{equation*}
The first transition changes the treatment while holding the mediator at its natural control value; the second changes the mediator to its natural treated value while keeping treatment fixed at $A=1$. Direct FNA counts units for whom the first transition is downward, and indirect FNA counts units for whom the second transition is downward. Total FNA instead compares only the first and last states.

Because FNA applies an indicator before averaging, there is no algebraic identity turning the three probabilities into an additive decomposition. For a given individual, the two path-specific changes can have opposite signs. For example, an individual can move from $1$ to $0$ on the first step and then from $0$ to $1$ on the second step: this unit contributes to direct FNA but not indirect FNA and is not harmed in total. Conversely, another individual can move from $0$ to $1$ and then $1$ to $0$: this unit contributes to indirect FNA but not direct FNA and is again not harmed in total. Hence, path-specific harms can cancel in the final outcome.

The two-subgroup construction in Example~3.1 makes this cancellation exact. In Group 1, the trajectory is $1\rightarrow0\rightarrow1$; in Group 2 it is $0\rightarrow1\rightarrow0$. Therefore every individual ends at the same level at which they started, implying $\fna_{\mathrm{tot}}=0$, but one half of the population is directly harmed and the other half is indirectly harmed, so $\fna_{\mathrm{dir}}=\fna_{\mathrm{ind}}=1/2$. Thus $\fna_{\mathrm{dir}}+\fna_{\mathrm{ind}}=1$ while $\fna_{\mathrm{tot}}=0$. The example is not a pathology of the estimator; it follows directly from the fact that the three quantities are probabilities of different within-unit events.

This nonadditivity also clarifies the interpretation of the decomposition. We do not claim that direct and indirect FNA are additive ``pieces'' whose sum recovers total FNA. Instead, they are complementary risk measures for adjacent stages of the counterfactual trajectory. A large indirect FNA can indicate that intervention effects transmitted through the mediator create substantial individual-level risk even when the final outcome looks favourable for many units; a large direct FNA indicates that substantial harm remains even when the mediator is held at its control-world value.

\subsection{Why there is no separate confounding-path term in the FNA decomposition}
\label{app:confounding-path}
The path $A\leftarrow X\rightarrow Y$ is not a causal path from $A$ to $Y$ under an intervention on $A$. Under $do(A=a)$, the incoming arrow into $A$ is cut, so variation in $X$ that predicts treatment assignment does not create an additional causal pathway from the intervention to the outcome. For an individual, the relevant FNA comparison is between counterfactual outcomes under two interventions on the same treatment variable; the background covariates and latent characteristics of that individual are held fixed across these counterfactual worlds.

This does not make back-door structure irrelevant. The role of $X$ in our framework is to support identification of the nested \emph{marginal} distributions through the treatment and mediator adjustment assumptions, and to tighten the Makarov bounds by conditioning on prognostic information. But it does not generate an additional ``spurious harm'' component in the decomposition of the treatment intervention itself.

This differs from the total-variation decompositions studied in causal fairness~\citep{plecko2024causalfairnessanalysis}, where the target is a \emph{group-level disparity}, for example a difference in outcome distributions or conditional expectations across values of a protected attribute. Such group-level quantities can be decomposed into direct, indirect, and spurious associations because the protected attribute may be statistically associated with the outcome through back-door paths. The object is different from FNA: FNA is an intervention-level, within-unit counterfactual harm event that is subsequently aggregated over individuals. Therefore, reproducing a fairness ``spurious'' component would require changing the causal contrast, not merely adding another term to the current FNA decomposition.

One can, of course, formulate a related but distinct harm question under an intervention on another variable, such as $X$. Such an estimand would concern the consequences of changing that variable, and any decomposition would be a different causal analysis. We therefore regard ``harm through $A\leftarrow X\rightarrow Y$'' as outside the present treatment-level decomposition rather than as a missing component of it.

\subsection{FNA as an individual event aggregated to a population fraction}
\label{app:fna-level}
There is no contradiction between calling the harm event ``individual-level'' and calling FNA a ``population-level'' quantity. The distinction is between the \emph{level of the primitive event} and the \emph{level at which it is summarized}. For each unit $i$, define the direct harm indicator
\begin{equation*}
H_i^{\mathrm{dir}}=\mathbbm{1}\{Y_i(1,M_i(0))<Y_i(0,M_i(0))\}.
\end{equation*}
The direct FNA is simply its population mean,
\begin{equation*}
\fna_{\mathrm{dir}}=\mathbb{E}[H_i^{\mathrm{dir}}].
\end{equation*}
The same construction applies to indirect and total harm. Thus, FNA is a population fraction, but what is being counted is whether each individual experiences a negative within-unit counterfactual comparison.

This same-unit pairing is precisely what creates the identification problem. An RCT identifies the marginal law of each potential outcome (or, under sequential ignorability, the relevant marginal law of each nested potential outcome), but it does not reveal which treated-world outcome belongs to the same unit as which control-world outcome. Consequently, the probability of the individual harm event is not generally point-identified. The Makarov construction exploits the identified marginals to derive the sharp range of possible values of this population fraction.

\subsection{Operational interpretation for harm-aware policy decisions}
\label{app:policy-implications}
The practical value of the decomposition is diagnostic as well as evaluative. A total FNA can indicate that an intervention may harm a nontrivial fraction of units, but it does not say whether the risk is concentrated in a mediator pathway that could plausibly be modified. Direct and indirect FNA provide separate pathway-specific risk measures, allowing a decision maker to ask whether a potentially harmful intervention should be redesigned by changing the mediator mechanism, by modifying the intervention itself, or by avoiding the policy for some covariate-defined population.

This perspective is compatible with harm-aware policy learning. For example, one can consider a policy objective that maximizes expected benefit subject to a constraint on a bound for indirect FNA, or compare candidate policies using a two-dimensional summary of average benefit and pathway-specific harm. The choice of constraint is application-dependent and normative; the framework does not imply that one particular harm threshold is universally appropriate. Rather, it makes the relevant trade-off explicit while acknowledging the partial identification of the underlying harm quantity.

\newpage

\section{Extended Related Work}\label{app:extendedrw}
\textbf{Causal mediation analysis} 
An important aspect in causal inference is understanding \emph{how} treatment effects propagate through different mechanisms. Causal mediation analysis formalizes this by decomposing total effects into direct and indirect components within the potential outcomes framework \citep{imai2010mediationgeneral}. Subsequent work has developed semiparametric efficiency theory and robust estimators for these quantities \citep{tchetgen2012semiparametric,Zheng2012TMLE,Farbmacher2022mediationdml,qi2026qrmediationconditional}. More generally, path-specific effects isolate contributions along subsets of causal pathways using nested counterfactuals such as $Y(a,M(a'))$ \citep{avin2005identifiability}. While this literature provides powerful tools for mechanism-level analysis, it primarily focuses on \emph{average} effects and does not address individual-level risk or harm. Our work complements this literature by addressing path-specific harm at the individual level through partial identification of direct and indirect FNA.

\textbf{Causal fairness}

Causal fairness studies how disparities arise through different causal pathways~\citep{kilbertus2017avoiding,zhang2018causalfairness,kusner2017counterfactual,imai2023principalfairness}, often decomposing total effects into direct and indirect components to distinguish sources of discrimination \citep{zhang2017causal,Chiappa2019pathcounterfactualfairness,plecko2024causalfairnessanalysis}. More research has been developed based on path-specific fairness \citep{chikahara2021fair,liu2023path,yao2023pathspecificfairapplication}, leveraging decompositions similar to those used in mediation analysis to attribute observed differences to underlying mechanisms.

Beyond this methodological connection, fairness and harm are closely related at a conceptual level.  If an intervention leads to worse outcomes for a non-trivial subset of individuals or groups, it is naturally viewed as unfair to that subpopulation~\citep{rudolph2021MOT,rudolph2022ends}. Conversely, many fairness criteria can be interpreted as restricting such adverse effects, ensuring that decisions do not disproportionately harm certain groups~\citep{kilbertus2017avoiding}.

Despite this overlap, their primary focus is different. Fairness is typically concerned with group-level disparities, whereas FNA notion of harm measures individual-level adverse effects which is only partially identifiable.

While causal fairness criteria (e.g. \cite{kilbertus2017avoiding,zhang2018causalfairness,kusner2017counterfactual,imai2023principalfairness}) focus on parity of decisions, a harm perspective shifts attention to the consequences of those decisions, asking whether deployment leads to adverse outcomes for particular subpopulations.
Within causal fairness, a similar line of work decomposes group-level disparities along causal pathways. \citet{zhang2017causal} formalize direct and indirect discrimination via path-specific effects, and subsequent work develops richer structural decompositions of total disparity \citep{Chiappa2019pathcounterfactualfairness, wu2019pcfairness, plecko2024causalfairnessanalysis}. These approaches mostly focus on average path-specific effects of a protected attribute, which are point-identified under standard assumptions. 

Our setting differs in two key aspects. First, we study harm induced by a substantive intervention rather than disparity induced by a protected attribute. Second, we focus on path-specific individual treatment harm which depends on the joint potential outcome distribution and is therefore not point-identified even under RCTs. This makes partial identification intrinsic to our problem.

\textbf{Other harm notions.}
Early work on harm is rooted in philosophy, where harm is defined via \emph{actual causality}~\citep{halpern2016actualcausality}: an action harms an individual if it is an actual cause of a worse outcome in a specific realised scenario. Subsequent work extends it to define qualitative and quantitative notions of harm based on utility comparisons \citep{beckers2022actualcausalharm, beckers2023quantifyingharm}. These approaches are primarily concerned with \emph{individual-level} and \emph{case-specific} harm (e.g., attributing responsibility).

Beyond actual causality, the literature on harm falls into two streams. The most straightforward is the \emph{interventional} criterion: a treatment is harmful if the ATE is negative, or more generally if $\mathbb{E}[u(Y(a_2)) - u(Y(a_1))] < 0$ for some utility $u$ that encodes asymmetric preferences over gains and losses. \citet{sarvet2025perspectives} provide a recent overview of such interventional harm notions. The second and most relevant stream takes a \emph{counterfactual} view: an individual is counterfactually harmed if their outcome under treatment is worse than it would have been under the alternative, i.e.\ $Y_i(a_1) < Y_i(a_0)$. Two related but distinct population aggregates of this individual predicate appear in the literature. The \emph{probability of necessity} (PN) \citep{pearl2001direct}, $\Pr(Y(a_0) > Y(a_1) \mid Y=y, A=a_1)$, conditions on the realized outcome and treatment, making it suited to retrospective attribution---answering whether a specific observed bad outcome was caused by the treatment \citep{richens2022counterfactualharm, Straitouri2024controllingCF}. The \emph{fraction negatively affected} (FNA) \citep{Mueller2022individualeffect}, $\Pr(Y(a_1) < Y(a_0))$, is the marginal probability of harm with no conditioning on realized outcomes, making it the natural risk metric for population-level policy evaluation \citep{Kallus2022FNA_sharp_bounds, BenMichael2024assymmetric}. Crucially, both depend on the \emph{joint} distribution of $(Y(a_0), Y(a_1))$, which is never simultaneously observed and hence only partially identified even in randomised experiments \citep{heckman1997heterogeneity, Fan2010covariatessharpbounds}. Our work adopts the FNA as its harm metric and, for the first time, decomposes it along direct and indirect causal pathways.

\textbf{Partial identification and Makarov bounds}

The problem of bounding functionals of unobserved joint distributions has its roots in classical results by \citet{Makarov1982} and \citet{ruschendorf1982}, who characterized sharp bounds on sums of random variables with fixed marginals. These ideas were later connected to treatment effect heterogeneity by \citet{heckman1997heterogeneity}, who emphasized the importance of the joint distribution of potential outcomes and derived early Fr\'echet-type bounds.

Subsequent work established Makarov bounds as a central tool for bounding the distribution of individual treatment effects. \citet{Fan2010covariatessharpbounds} showed that these bounds are sharp in this setting and demonstrated that conditioning on covariates yields weakly tighter intervals. Later work refined these results by incorporating additional structure, such as improved bounds using moment information \citep{FIRPO2019tightenbounds}, stochastic monotonicity assumptions \citep{frandsen2021partial}, and extensions to conditional treatment effect distributions with valid inference \citep{lee2023partialidentificationinferenceconditional}.

More recently, there has been growing interest in combining partial identification with modern machine learning tools. \citet{melnychuk2024makarovbound} develop orthogonal learners for Makarov bounds under flexible nuisance estimation, and \citet{Kallus2022FNA_sharp_bounds} applies this framework specifically to the FNA, providing sharp bounds and doubly robust inference.

Despite these advances, existing work applies Makarov bounds exclusively to \emph{total} treatment effects, such as $\Pr(Y(1) < Y(0))$. In contrast, our setting involves \emph{path-specific} harm, which depends on nested potential outcomes $Y(a, M(a'))$. Bounding such quantities requires an additional layer of structure. This combination of mediation analysis and partial identification has not been studied in the literature.

\textbf{Research gap}
To the best of our knowledge, we are the first to study path-specific harm and propose orthogonal learners for estimating Makarov bounds for direct and indirect FNA. Our work bridges the gaps between harm quantification, mediation analysis, and partial identification by: (1)~formalizing direct and indirect harm through FNA decomposition; (2) deriving sharp Makarov bounds for path-specific FNA using identified nested marginals; and (3) developing semiparametrically efficient estimators that accommodate modern machine learning while providing valid inference. This enables principled decision-making that accounts for harm along specific causal pathways rather than only aggregate effects.

\newpage

\section{Sharp Bounds for Direct and Indirect FNA}
\label{app:sharp_bounds}

\begin{proposition}[Conditional Makarov bounds for path-specific FNA]
\label{prop:conditional_makarov}
Let $(U,V)$ denote one of the path-specific pairs of nested potential outcomes:
\begin{equation}
(U,V)=(Y(1,M(0)),Y(0,M(0)))
\end{equation}
for direct FNA, or
\begin{equation}
(U,V)=(Y(1,M(1)),Y(1,M(0)))
\end{equation}
for indirect FNA. Suppose the conditional marginal CDFs
$F_U(t\mid X)$ and $F_V(t\mid X)$ are identified and admit regular conditional distributions. Then the sharp covariate-assisted bounds for $\theta=\Pr(U<V)$ are
\begin{equation}
\underline{\theta}
=
\mathbb E\left[
\left(
\sup_{t\in\mathbb R}
\{F_U(t\mid X)-F_V(t\mid X)\}
\right)_+
\right],
\end{equation}
and
\begin{equation}
\overline{\theta}
=
1-
\mathbb E\left[
\left(
\sup_{t\in\mathbb R}
\{F_V(t\mid X)-F_U(t\mid X)\}
\right)_+
\right].
\end{equation}
\end{proposition}

\begin{proof}
Fix \(x\) in a set of probability one on which the regular conditional
distributions of \(U\) and \(V\) given \(X=x\) exist. Conditional on
\(X=x\), the marginal CDFs of \(U\) and \(V\) are \(F_U(\cdot\mid x)\)
and \(F_V(\cdot\mid x)\), while their conditional joint distribution is
otherwise unrestricted. Hence the Makarov bounds imply
that
\begin{equation}
\ell(x)
=
\left[
\sup_t\{F_U(t\mid x)-F_V(t\mid x)\}
\right]_+
\end{equation}
and
\begin{equation}
u(x)
=
1-
\left[
\sup_t\{F_V(t\mid x)-F_U(t\mid x)\}
\right]_+
\end{equation}
are the sharp conditional bounds for
\(\Pr(U<V\mid X=x)\), assuming conditional continuity of the outcome
distributions.\footnote{For general distributions, the corresponding left-limit
version is used for strict inequalities.}

Therefore, for any joint distribution consistent with the identified
conditional marginals,
\begin{equation}
\Pr(U<V)
=
\mathbb E\{\Pr(U<V\mid X)\}
\in
\left[
\mathbb E\{\ell(X)\},\mathbb E\{u(X)\}
\right].
\end{equation}

It remains to show sharpness after averaging. For almost every \(x\), let
\(Q_x^L\) and \(Q_x^U\) denote conditional couplings of \(U\) and \(V\)
with the prescribed conditional marginals that attain \(\ell(x)\) and
\(u(x)\), respectively. Under the usual measurable-selection regularity
conditions\footnote{Variables take values in standard Borel
spaces, so that regular conditional distributions exist}, these couplings may be chosen measurably in \(x\). For any
\(\theta\in[\mathbb E\ell(X),\mathbb E u(X)]\), define
\begin{equation}
\lambda
=
\frac{\theta-\mathbb E\{\ell(X)\}}
{\mathbb E\{u(X)\}-\mathbb E\{\ell(X)\}},
\end{equation}
with arbitrary \(\lambda\in[0,1]\) if the denominator is zero. Conditional
on \(X=x\), take the mixture coupling
\begin{equation}
Q_x^\lambda
=
(1-\lambda)Q_x^L+\lambda Q_x^U.
\end{equation}
This coupling preserves the conditional marginals \(F_U(\cdot\mid x)\) and
\(F_V(\cdot\mid x)\), and yields
\begin{equation}
\Pr\nolimits_{Q^\lambda}(U<V\mid X=x)
=
(1-\lambda)\ell(x)+\lambda u(x).
\end{equation}
Averaging over \(X\) gives \(\Pr_{Q^\lambda}(U<V)=\theta\). Hence every
value in the displayed interval is attainable, so the bounds are sharp.
\end{proof}
Applying Proposition~\ref{prop:conditional_makarov} to the direct and indirect path-specific pairs yields
\begin{equation}
\underline{\fna}_{\mathrm{dir}}
=
\E\left[
\left(
\sup_t\{F_{1,0}(t\mid X)-F_{0,0}(t\mid X)\}
\right)_+
\right],
\end{equation}
\begin{equation}
\overline{\fna}_{\mathrm{dir}}
=
1-
\E\left[
\left(
\sup_t\{F_{0,0}(t\mid X)-F_{1,0}(t\mid X)\}
\right)_+
\right],
\end{equation}
and
\begin{equation}
\underline{\fna}_{\mathrm{ind}}
=
\E\left[
\left(
\sup_t\{F_{1,1}(t\mid X)-F_{1,0}(t\mid X)\}
\right)_+
\right],
\end{equation}
\begin{equation}
\overline{\fna}_{\mathrm{ind}}
=
1-
\E\left[
\left(
\sup_t\{F_{1,0}(t\mid X)-F_{1,1}(t\mid X)\}
\right)_+
\right].
\end{equation}

\newpage

\section{Efficient Influence Functions}
\label{app:eif}

We derive the efficient influence functions for the lower and upper conditional Makarov bounds. Throughout this section, let $(U,V)$ denote a generic path-specific pair and let
\begin{equation}
\Delta_X(t)=F_U(t\mid X)-F_V(t\mid X).
\end{equation}

\textbf{Orthogonal pseudo outcome for nested CDFs.}\label{app:orthogonalscore}
For each pair $(a,a')$, define
\begin{equation}\label{eqn:pseudooutcome}
\Gamma_{a,a'}^t(O)
=
\frac{\mathbbm 1(A=a)}{\pi_a(X)}
\frac{g_{a'}(M\mid X)}{g_a(M\mid X)}
\{Y_t-\mu_a(t,M,X)\}
+
\frac{\mathbbm 1(A=a')}{\pi_{a'}(X)}
\{\mu_a(t,M,X)-\eta_{a,a'}(t,X)\}
+
\eta_{a,a'}(t,X).
\end{equation}
For each fixed threshold \(t\), the following lemma states that this is a Neyman-orthogonal pseudo-outcome whose conditional
mean equals \(F_{a,a'}(t\mid X)\). Equivalently, its marginal expectation is
the uncentered efficient score for \(F_{a,a'}(t)\).
\begin{lemma}[Orthogonal pseudo-outcome for nested CDFs]\label{lem:orthogonal_pseudo_outcome}
For each fixed threshold \(t\), the pseudo-outcome $\Gamma_{a,a'}^t(O)$ defined in Eq.~\eqref{eqn:pseudooutcome} satisfies
$$\E\{\Gamma_{a,a'}^t(O)\mid X\}
=
F_{a,a'}(t\mid X).$$

Consequently, for any measurable threshold rule $t(X)$ depending on $X$,

$$\E\{\Gamma_{a,a'}^{t(X)}(O)\mid X\}
=
F_{a,a'}(t(X)\mid X).$$

Moreover, the corresponding centered influence function for $\Psi_{a,a'}(t)=\E\{F_{a,a'}(t(X)\mid X)\}$ 
is $\Gamma_{a,a'}^{t(X)}(O)-\Psi_{a,a'}(t)$.

\end{lemma}

\begin{proof}
We first prove the identity for a fixed threshold \(t\in\mathcal T\). Taking
conditional expectation given \(X\), the first term in Eq.~\eqref{eqn:pseudooutcome}satisfies
$$\E\!\left[
\frac{\mathbbm 1(A=a)}{\pi_a(X)}
\frac{g_{a'}(M\mid X)}{g_a(M\mid X)}
\{Y_t-\mu_a(t,M,X)\}
\;\middle|\; X
\right]
=0,$$
because $\E\{Y_t-\mu_a(t,M,X)\mid X,M,A=a\}=0$. Indeed, conditioning further on \((X,M,A=a)\) makes the residual mean zero, and
the factor in front is measurable with respect to \((X,M,A)\).

For the second term,
\[
\E\!\left[
\frac{\mathbbm 1(A=a')}{\pi_{a'}(X)}
\{\mu_a(t,M,X)-\eta_{a,a'}(t,X)\}
\;\middle|\; X
\right]
=
\int
\{\mu_a(t,m,X)-\eta_{a,a'}(t,X)\}
g_{a'}(m\mid X)\,dm.
\]
By definition,
\[
\eta_{a,a'}(t,X)
=
\int \mu_a(t,m,X)g_{a'}(m\mid X)\,dm,
\]
so the last term equals \(0\). Therefore,
\[
\E\{\Gamma_{a,a'}^t(O)\mid X\}
=
\eta_{a,a'}(t,X)
=
F_{a,a'}(t\mid X).
\]

Now let \(t(X):\mathcal X\to\mathcal T\) be any measurable threshold rule.
Since \(t(X)\) is measurable with respect to \(X\), we may evaluate the
preceding identity pointwise at \(t=\tau(X)\). Thus,
\[
\E\{\Gamma_{a,a'}^{\tau(X)}(O)\mid X\}
=
F_{a,a'}(\tau(X)\mid X).
\]
We now verify Neyman orthogonality. Fix \(t,a,a'\), and condition throughout on \(X=x\). Let $\nu=(\pi_a,\pi_{a'},g_a,g_{a'},\mu_a)$
be the nuisance functions. For a generic nuisance value
\(\bar\nu=(\bar\pi_a,\bar\pi_{a'},\bar g_a,\bar g_{a'},\bar\mu_a)\), define
$\bar\eta_{a,a'}(t,x)
=
\int \bar\mu_a(t,m,x)\bar g_{a'}(m\mid x)\,d\lambda(m)$.

Consider the conditional moment map
\[
\Psi_x(\bar\nu)
=
\E\left[
\frac{\mathbbm 1(A=a)}{\bar\pi_a(x)}
\frac{\bar g_{a'}(M\mid x)}{\bar g_a(M\mid x)}
\{Y_t-\bar\mu_a(t,M,x)\}
\;\middle|\; X=x
\right]
\]
\[
\quad+
\E\left[
\frac{\mathbbm 1(A=a')}{\bar\pi_{a'}(x)}
\{\bar\mu_a(t,M,x)-\bar\eta_{a,a'}(t,x)\}
\;\middle|\; X=x
\right]
+
\bar\eta_{a,a'}(t,x).
\]
At the truth \(\nu_0=(\pi_a,\pi_{a'},g_a,g_{a'},\mu_a)\),
\(\Psi_x(\nu_0)=\eta_{a,a'}(t,x)\), because
\[
\E\{Y_t-\mu_a(t,M,X)\mid A=a,M,X\}=0
\]
and
\[
\E\{\mu_a(t,M,x)-\eta_{a,a'}(t,x)\mid A=a',X=x\}
=
\int\{\mu_a(t,m,x)-\eta_{a,a'}(t,x)\}g_{a'}(m\mid x)d\lambda(m)
=0.
\]

Let \(\nu_\varepsilon\) be a regular one-dimensional nuisance path through the
truth with tangent directions
\[
h_{\pi_a}(x)
=
\left.\frac{d}{d\varepsilon}\pi_{a,\varepsilon}(x)\right|_{\varepsilon=0},
\quad
h_{g_a}(m\mid x)
=
\left.\frac{d}{d\varepsilon}g_{a,\varepsilon}(m\mid x)\right|_{\varepsilon=0},
\quad
h_\mu(t,m,x)
=
\left.\frac{d}{d\varepsilon}\mu_{a,\varepsilon}(t,m,x)\right|_{\varepsilon=0}.
\] 
Because \(g_{a,\varepsilon}(\cdot\mid x)\) is a conditional density, its tangent
direction satisfies~\footnote{All integrals over \(m\) are with respect to a dominating measure 
\(\lambda\). For continuous \(M\), this is Lebesgue measure. For binary \(M\), 
\(\lambda\) is counting measure on \(\{0,1\}\), so the integral becomes 
\(\sum_{m\in\{0,1\}}\). In that case, the mediator tangent direction satisfies 
\(\sum_{m\in\{0,1\}}h_{g_b}(m\mid x)=0\).}
\[
\int h_{g_a}(m\mid x)\,dm=0. \qquad a\in\{a,a'\}.
\]
For binary treatment, the treatment-law perturbations satisfy $h_{\pi_0}(x)+h_{\pi_1}(x)=0$.

Write
\[
\pi_{a,\varepsilon}(x)
=
\pi_a(x)+\varepsilon h_{\pi_a}(x)+o(\varepsilon),
\qquad
\pi_{a',\varepsilon}(x)
=
\pi_{a'}(x)+\varepsilon h_{\pi_{a'}}(x)+o(\varepsilon),
\]
\[
g_{a,\varepsilon}(m\mid x)
=
g_a(m\mid x)+\varepsilon h_{g_a}(m\mid x)+o(\varepsilon),
\qquad
g_{a',\varepsilon}(m\mid x)
=
g_{a'}(m\mid x)+\varepsilon h_{g_{a'}}(m\mid x)+o(\varepsilon),
\]
and
\[
\mu_{a,\varepsilon}(t,m,x)
=
\mu_a(t,m,x)+\varepsilon h_{\mu_a}(t,m,x)+o(\varepsilon).
\]
Along the path,
\[
\eta_{\varepsilon}
:=
\eta_{a,a',\varepsilon}(t,x)
=
\int
\mu_{a,\varepsilon}(t,m,x)g_{a',\varepsilon}(m\mid x)d\lambda(m),
\]
so
\[
\dot\eta
:=
\left.\frac{d}{d\varepsilon}\eta_{\varepsilon}\right|_{\varepsilon=0}
=
\int h_{\mu_a}(t,m,x)g_{a'}(m\mid x)d\lambda(m)
+
\int \mu_a(t,m,x)h_{g_{a'}}(m\mid x)d\lambda(m).
\]

We need to show
\[
\left.
\frac{d}{d\varepsilon}
\{\Psi_x(\nu_\varepsilon)-\eta_{a,a'}(t,x)\}
\right|_{\varepsilon=0}
=0.
\]
Since \(\eta_{a,a'}(t,x)\) is fixed at the truth, it suffices to differentiate
\(\Psi_x(\nu_\varepsilon)\). 
Decompose $\Psi_x(\nu_\varepsilon) = T_1(\varepsilon)+T_2(\varepsilon)+T_3(\varepsilon)$ where $$T_1(\varepsilon)
=
\frac{\pi_a(x)}{\pi_{a,\varepsilon}(x)}
\int
\frac{g_{a',\varepsilon}(m\mid x)}{g_{a,\varepsilon}(m\mid x)}
\{\mu_a(t,m,x)-\mu_{a,\varepsilon}(t,m,x)\}
g_a(m\mid x)d\lambda(m),$$
\[
T_2(\varepsilon)
=
\frac{\pi_{a'}(x)}{\pi_{a',\varepsilon}(x)}
\int
\{\mu_{a,\varepsilon}(t,m,x)-\eta_\varepsilon\}
g_{a'}(m\mid x)d\lambda(m),
\]
and 
\[T_3(\varepsilon)=\eta_\varepsilon(t,x).\] 

\underline{For $T_1$:} At $\varepsilon=0$, the integrand in $T_1$ equals $\frac{g_{a'}(m\mid x)}{g_{a}(m\mid x)}
\{\mu_a(t,m,x)-\mu_{a}(t,m,x)\}
g_a(m\mid x)=0$, so the first-order contributions from perturbing $\pi_{a,\varepsilon}, g_{a,\varepsilon} and g_{a',\varepsilon}$ all multiply this zero residual and vanish. Only the perturbation to $\mu_a{,\varepsilon}$ survives, which gives
\[
T_1'(0)
=
-\int h_{\mu_a}(t,m,x)g_{a'}(m\mid x)d\lambda(m).
\]

\underline{For $T_2$:} Since $g_{a'}$ in $T_2$ is the unperturbed density, the only perturbed quantities are the propensity prefactor and the integrand $\mu_{a,\varepsilon}-\eta_{\varepsilon}$. At $\varepsilon=0$,
\[
\int\{\mu_a(t,m,x)-\eta_{a,a'}(t,x)\}g_{a'}(m\mid x)d\lambda(m)=0,
\]
the derivative of the prefactor \(\pi_{a'}(x)/\pi_{a',\varepsilon}(x)\) multiplies zero and contributes
nothing at first order. Differentiating the integrand directly, we get 
$$\int h_{\mu_a}(t,m,x)\,g_{a'}(m\mid x)\,d\lambda(m)
-\dot\eta,$$ where the $-\dot\eta$ comes from differentiating $-\eta_\varepsilon$ inside the integrand.
Hence
\[
T_2'(0)
=
\int h_{\mu_a}(t,m,x)g_{a'}(m\mid x)d\lambda(m)
-
\dot\eta.
\]
\underline{For $T_3$:}
$T_3'(0)=\dot\eta=\int h_{\mu_a}(t,m,x)g_{a'}(m\mid x)d\lambda(m)+\int \mu_a(t,m,x)h_{g_{a'}}(m\mid x)d\lambda(m)$.
This accounts for perturbations in both $\mu_{a,\varepsilon}$ and $g_{a',\varepsilon}$ within $\eta_{\varepsilon}$. 

Combining the three derivatives,
\[
\left.
\frac{d}{d\varepsilon}\Psi_x(\nu_\varepsilon)
\right|_{\varepsilon=0}
=
-\int h_{\mu_a}g_{a'}d\lambda
+
\left\{\int h_{\mu_a}g_{a'}d\lambda-\dot\eta\right\}
+
\dot\eta
=0.
\]
Therefore,
\[
\left.
\frac{d}{d\varepsilon}
\{\Psi_x(\nu_\varepsilon)-\eta_{a,a'}(t,x)\}
\right|_{\varepsilon=0}
=0,
\]
which establishes Neyman orthogonality.

\end{proof}

\begin{corollary}[Pseudo-outcome identity for the CDF contrast]
\label{cor:delta-pseudo-outcome}
For a generic pair \((U,V)\), define
\[
\Gamma_\Delta^t(O)=\Gamma_U^t(O)-\Gamma_V^t(O).
\]
Then, for any measurable threshold rule \(\tau(X)\),
\[
\E\{\Gamma_\Delta^{\tau(X)}(O)\mid X\}
=
\Delta_X(\tau(X)).
\]
\end{corollary}

\subsection{Proof of Theorem~\ref{thm:eif_bounds}}
The orthogonality above holds pointwise in \(t\). Since the FNA bounds involve
a supremum over \(t\), the next theorem additionally requires regularity of
the supremum map and stability of the estimated maximisers.

\begin{assumption}[Margin and non-kink condition]
\label{ass:margin}
For $b\in\{L,U\}$, let $h_L(z)=z$ and $ h_U(z)=-z$, and define $t_b(X)\in\arg\max_{t\in\mathbb R} h_b\{\Delta_X(t)\}$.

For $P_X$-almost every $X=x$, the maximiser $t_b(x)$ is unique. Moreover, there exist neighborhoods $\mathcal N_b(x)$ and constants $c>0$, $\kappa>0$ such that, for all $t\in\mathcal N_b(x)$,
\begin{equation}
h_b\{\Delta_x(t_b(x))\}-h_b\{\Delta_x(t)\}
\ge c |t-t_b(x)|^\kappa .
\end{equation}
Finally, the positive-part kink is avoided:
\begin{equation}
\Pr\left\{\sup_t \Delta_X(t)=0\right\}=0,
\qquad
\Pr\left\{\sup_t -\Delta_X(t)=0\right\}=0.
\end{equation}
\end{assumption}

\begin{theorem}[EIFs and efficiency for conditional direct/indirect FNA bounds]
Consider either 
$\Delta_{\mathrm{dir},X}(t)=F_{1,0}(t\mid X)-F_{0,0}(t\mid X)$
or
$\Delta_{\mathrm{ind},X}(t)=F_{1,1}(t\mid X)-F_{1,0}(t\mid X)$. Under the margin and non-kink condition~\ref{ass:margin}, the conditional Makarov bounds in \eqref{eq:conditional_makarov_bounds_general} are pathwise differentiable, with efficient influence functions
\begin{align}
\phi_{\underline{\theta}}(O)
&=
\mathbbm 1\{\sup_{t}\Delta_X(t)>0\}\,
\Gamma_{\Delta}^{\,\underline{t}_X}(O)
-\underline{\theta},
\\
\phi_{\overline{\theta}}(O)
&=
1+
\mathbbm 1\{\sup_{t}(-\Delta_X(t))>0\}\,
\Gamma_{\Delta}^{\,\overline{t}_X}(O)
-\overline{\theta},
\end{align}
where $\underline{t}_X \in \arg\max_{t\in\mathbb R}\Delta_X(t)$ and $\overline{t}_X \in \arg\max_{t\in\mathbb R}\{-\Delta_X(t)\}$, and $\Gamma_{\Delta}^{\cdot}$ is defined in~\eqref{eqn:deltascore}.

\end{theorem}

\begin{proof}
We prove the result for the lower bound. The upper bound is analogous. Recall the lower bound $\underline{\theta}
=
\E\left[
\left\{\sup_{t\in\mathcal T}\Delta_X(t)\right\}_+
\right]$.

Define
\[
m_L(X)=\sup_{t\in\mathcal T}\Delta_X(t),
\qquad
\underline t_X\in\arg\max_{t\in\mathcal T}\Delta_X(t).
\]
Then
\[
\underline\theta
=
\E\{m_L(X)_+\}.
\]
By Assumption~\ref{ass:margin}, the maximiser \(\underline t_X\) is unique and
well separated for \(P_X\)-almost every \(X\). Hence the envelope theorem gives
that the first-order perturbation of \(m_L(X)\) is obtained by evaluating the
first-order perturbation of \(\Delta_X(t)\) at \(t=\underline t_X\). Moreover,
by the non-kink condition,
\[
\Pr\{m_L(X)=0\}=0,
\]
so the positive-part map is differentiable at \(m_L(X)\) almost surely.

Consequently, the first-order perturbation of the lower bound is the same as
the first-order perturbation of the fixed-threshold parameter
\[
\E\left[
\mathbbm 1\{m_L(X)>0\}
\Delta_X(\underline t_X)
\right].
\]
Equivalently,
\[
\underline\theta
=
\E\left[
\mathbbm 1\{m_L(X)>0\}
\Delta_X(\underline t_X)
\right].
\]

By Corollary~\ref{cor:delta-pseudo-outcome}, applied with
\(\tau(X)=\underline t_X\),
\[
\E\{\Gamma_\Delta^{\underline t_X}(O)\mid X\}
=
\Delta_X(\underline t_X).
\]
Therefore an uncentered influence-function representation for the lower bound is
\[
\mathbbm 1\{m_L(X)>0\}
\Gamma_\Delta^{\underline t_X}(O).
\]
Centering gives
\[
\phi_{\underline\theta}(O)
=
\mathbbm 1\{m_L(X)>0\}
\Gamma_\Delta^{\underline t_X}(O)
-
\underline\theta.
\]
Indeed,
\[
\E\{\phi_{\underline\theta}(O)\}
=
\E\left[
\mathbbm 1\{m_L(X)>0\}
\Delta_X(\underline t_X)
\right]
-
\underline\theta
=
0.
\]
Thus, for every regular parametric submodel with score \(s\in L_0^2(P)\),
\[
\left.
\frac{d}{d\epsilon}
\underline\theta(P_\epsilon)
\right|_{\epsilon=0}
=
\E\{\phi_{\underline\theta}(O)s(O)\}.
\]

For the upper bound, define
\[
m_U(X)=\sup_{t\in\mathcal T}\{-\Delta_X(t)\},
\qquad
\overline t_X\in\arg\max_{t\in\mathcal T}\{-\Delta_X(t)\}.
\]
Then
\[
\overline\theta
=
1-\E\{m_U(X)_+\}
=
1+
\E\left[
\mathbbm 1\{m_U(X)>0\}
\Delta_X(\overline t_X)
\right].
\]
Applying the same argument gives
\[
\phi_{\overline\theta}(O)
=
1+
\mathbbm 1\{m_U(X)>0\}
\Gamma_\Delta^{\overline t_X}(O)
-
\overline\theta.
\]
The mean-zero property follows since
\[
\E\{\phi_{\overline\theta}(O)\}
=
1+
\E\left[
\mathbbm 1\{m_U(X)>0\}
\Delta_X(\overline t_X)
\right]
-
\overline\theta
=
0.
\]

Since $\phi_{\underline{\theta}}$ and 
$\phi_{\overline{\theta}}$ are mean zero, square-integrable, and represent the 
pathwise derivatives of $\underline{\theta}$ and $\overline{\theta}$ 
respectively along every regular parametric submodel, they are the unique 
elements of $L_0^2(P)$ satisfying this gradient representation, and hence are 
the efficient influence functions for $\underline{\theta}$ and 
$\overline{\theta}$ in the nonparametric model.

\end{proof}

\newpage

\section{Asymptotic Normality of the One-Step Estimators}\label{app:asymptotic_normality}
For DML inference, we assume standard cross-fitting and rate conditions ensuring second-order remainder control. Concretely, we require:
(i) consistent nuisance estimators with overlap-safe truncation,
(ii) $L_2(P)$ convergence rates such that products of relevant errors are $o_p(n^{-1/2})$,
and (iii) empirical process conditions compatible with cross-fitting. 
In addition, we state the following condition on top of Assumption~\ref{ass:margin}.
\begin{assumption}[Estimated maximiser stability]
\label{ass:max_stability}
Define ${\underline{h}}(z)=z$ and ${\overline{h}}(z)=-z,$
so that
\begin{equation}
  \underline{t}(X)\in\arg\max_t\, {\underline{h}}\{\Delta_X(t)\},
  \qquad
  \overline{t}(X)\in\arg\max_t\, {\overline{h}}\{\Delta_X(t)\}.
\end{equation}
Let $\hat{\underline{t}}(X)$ and $\hat{\overline{t}}(X)$ be the corresponding
data-adaptive maximisers of ${\underline{h}}\{\widehat\Delta_X(t)\}$ and
${\overline{h}}\{\widehat\Delta_X(t)\}$ over $t\in\mathcal T_n$, computed on an
auxiliary fold independent of the final score evaluation.
For $b\in\{\underline{\phantom{t}},\overline{\phantom{t}}\}$, assume:
\begin{enumerate}[label=(\Alph*)]
  \item \emph{Oracle value loss.}
        \begin{equation}
          \mathbb E\!\left[
            \left|
              h_b\{\Delta_X(\hat{t}_b(X))\}
              -h_b\{\Delta_X(t_b(X))\}
            \right|
          \right]=o_p(n^{-1/2}).
          \tag{A}\label{eq:value_loss}
        \end{equation}
  \item \emph{Score-weighted switch stability.}
        \begin{equation}
          \mathbb E\!\left[
            \left|
              \mathbbm 1\{h_b(\widehat\Delta_X(\hat{t}_b(X)))>0\}
              -\mathbbm 1\{h_b(\Delta_X(t_b(X)))>0\}
            \right|
            \sup_{t\in\mathcal T_n}|\Gamma_\Delta^{\,t}(O)|
          \right]=o_p(n^{-1/2}).
          \tag{B}\label{eq:switch_stability}
        \end{equation}
  \item \emph{Local stochastic equicontinuity.}
        \begin{equation}
          (\mathbb P_n-P_0)\!\left[
            \mathbbm 1\{h_b(\Delta_X(t_b(X)))>0\}
            \left\{
              \Gamma_\Delta^{\,\hat{t}_b(X)}(O)
              -\Gamma_\Delta^{\,t_b(X)}(O)
            \right\}
          \right]=o_p(n^{-1/2}).
          \tag{C}\label{eq:equicont}
        \end{equation}
\end{enumerate}
\end{assumption}

\begin{remark}[Why Assumption~\ref{ass:max_stability} is needed]
The estimator evaluates the orthogonal score at the estimated maximiser $\hat t_b(X)$: $\widehat\Gamma_\Delta^{\hat t_b(X)}(O)$. Even if the nuisance functions entering $\widehat\Gamma_\Delta$ were known, using $\hat t_b(X)$ instead of the oracle maximiser $t_b(X)$ would target
$\mathbb E\left[h_b\{\Delta_X(\hat t_b(X))\}\right]$
rather than $\mathbb E\left[h_b\{\Delta_X(t_b(X))\}\right]$.
Thus, the plug-in maximiser can induce a first-order bias unless the oracle value loss in~\eqref{eq:value_loss} is $o_p(n^{-1/2})$. Uniqueness of the maximiser alone is not enough; the estimated maximiser must be sufficiently stable for root-$n$ inference.
\end{remark}

\begin{remark}[Sufficient conditions for maximiser stability]
Assumption~\ref{ass:max_stability} is high-level but can be verified under simpler primitive conditions.

\emph{Finite-grid separated maximiser.}
Suppose $\mathcal T_n=\mathcal T$ is a fixed finite grid and, for $b\in\{\underline{\phantom{t}},\overline{\phantom{t}}\}$, the oracle grid maximiser is separated:
\begin{equation}
\inf_x
\left[
h_b\{\Delta_x(t_b(x))\}
-
\max_{t\in\mathcal T:t\neq t_b(x)}h_b\{\Delta_x(t)\}
\right]
\ge c_0>0.
\end{equation}
If
\begin{equation}
\sup_{x,t\in\mathcal T}
|\widehat\Delta_x(t)-\Delta_x(t)|=o_p(c_0),
\end{equation}
then $\hat t_b(X)=t_b(X)$ with probability tending to one. All three conditions follow immediately: \eqref{eq:value_loss} holds because the value loss is exactly zero with high probability; \eqref{eq:switch_stability} holds because the two indicators agree; and \eqref{eq:equicont} holds because the score difference is zero.

\emph{Continuous-threshold value-loss condition.}
Suppose the local margin condition in Assumption~\ref{ass:margin} holds and $\mathbb E|\hat t_b(X)-t_b(X)|^\kappa=o_p(n^{-1/2})$.
The margin condition then implies
$|h_b\{\Delta_X(\hat{t}_b)\}-h_b\{\Delta_X(t_b)\}|\lesssim|\hat{t}_b-t_b|^\kappa$,
so \eqref{eq:value_loss} follows. Conditions \eqref{eq:switch_stability} and \eqref{eq:equicont} follow under mild Lipschitz regularity of the score in $t$ near the oracle maximiser.
Note that $\kappa$-th moment convergence of the maximiser is strictly stronger
than ordinary consistency: the estimator can remain biased even when
$\hat{t}_b(X)\to t_b(X)$ in probability if convergence is insufficiently fast. This condition can hold when the maximiser is estimated sufficiently accurately, for example in low-dimensional or strongly smoothed settings.
\end{remark}

\begin{theorem}[DML inference for conditional direct/indirect FNA bounds]
Under the conditional margin assumption and standard DML rate conditions with cross-fitting,
\begin{equation}
\sqrt{n}\big(\widehat{\overline{\underline{\theta}}}-{\overline{\underline{\theta}}}\big)
=
\frac{1}{\sqrt{n}}\sum_{i=1}^n
\phi_{\overline{\underline{\theta}}}(O_i)
+
o_p(1).
\end{equation}
Consequently,
\begin{equation}
\frac{\widehat{\overline{\underline{\theta}}}-{\overline{\underline{\theta}}}}
{\widehat{\mathrm{se}}({\overline{\underline{\theta}}})}
\ \overset{d}{\longrightarrow}\ 
\mathcal N(0,1).
\end{equation}
Furthermore, \(\widehat{\overline{\underline{\theta}}}\) is asymptotically efficient.
A \((1-\alpha)\) Wald interval for each bound is
\begin{equation}
\overline{\underline{\mathrm{CI}}}_{1-\alpha}
=
\left[
\widehat{\overline{\underline{\theta}}}
\pm
z_{1-\alpha/2}\,
\widehat{\mathrm{se}}(\widehat{\overline{\underline{\theta}}})
\right]\cap[0,1].
\end{equation}
A conservative confidence set for the partially identified FNA parameter 
\(\theta\in[\underline{\theta},\overline{\theta}]\) is
\begin{equation}
\left[
\max\{0,\ \widehat{\underline{\theta}}-z_{1-\alpha/2}\widehat{\mathrm{se}}(\widehat{\underline{\theta}})\},
\ 
\min\{1,\ \widehat{\overline{\theta}}+z_{1-\alpha/2}\widehat{\mathrm{se}}(\widehat{\overline{\theta}})\}
\right].
\end{equation}
\end{theorem}
\begin{proof}
We give the proof for the lower bound. The upper bound follows by replacing $\Delta_X(t)$ with $-\Delta_X(t)$ and accounting for the leading constant $1$.

\textbf{Step 0: Decomposition.}
Let $\underline{\psi}(O)
=
\mathbbm 1\{\Delta_X({\underline{t}}(X))>0\}
\Gamma_\Delta^{{\underline{t}}(X)}(O)$.
Note $\underline\theta = \mathbb{E}\{\psi_L(O)\}$ by 
Corollary~\ref{cor:delta-pseudo-outcome} and the envelope theorem. 
Adding and subtracting $\underline{\psi}$, the estimation error decomposes as
\[
  \widehat{\underline\theta} - \underline\theta
  = (\mathbb{P}_n - P_0)\phi_{\underline\theta}
  + R_{1n} + R_{2n} + R_{3n},
\]
where
\begin{align*}
R_{1n} &= \mathbb{P}_n\!\left[
  \mathbbm 1\{\widehat\Delta_X(\hat{\underline{t}})>0\}
  \Bigl(\widehat\Gamma_\Delta^{\hat{\underline{t}}}(O)
        - \Gamma_\Delta^{\hat{\underline{t}}}(O)\Bigr)
\right], \\
R_{2n} &= \mathbb{P}_n\!\left[
  \Bigl(\mathbbm 1\{\widehat\Delta_X(\hat{\underline{t}})>0\}
       -\mathbbm 1\{\Delta_X(\underline{t})>0\}\Bigr)
  \Gamma_\Delta^{\hat{\underline{t}}}(O)
\right], \\
R_{3n} &= \mathbb{P}_n\!\left[
  \mathbbm 1\{\Delta_X(\underline{t})>0\}
  \Bigl(\Gamma_\Delta^{\hat{\underline{t}}}(O)
        - \Gamma_\Delta^{\underline{t}}(O)\Bigr)
\right].
\end{align*}
We show each remainder is $o_p(n^{-1/2})$.

There is also a deterministic value-loss component hidden in $R_{3n}$:
\begin{equation}
P_0\left[
\mathbbm 1\{\Delta_X({\underline{t}})>0\}
\left\{
\Gamma_\Delta^{\hat {\underline{t}}}(O)
-
\Gamma_\Delta^{{\underline{t}}}(O)
\right\}
\right]
\end{equation}
\begin{equation}
=
\mathbb E\left[
\mathbbm 1\{\Delta_X({\underline{t}})>0\}
\left\{
\Delta_X(\hat {\underline{t}})-\Delta_X({\underline{t}})
\right\}
\right].
\end{equation}
This is exactly the bias induced by evaluating the score at the estimated maximiser. Assumption~\ref{ass:max_stability} requires this term to be $o_p(n^{-1/2})$.

\noindent\textbf{Bounding $R_{1n}$: nuisance remainder.}
Since the indicator $\mathbbm 1\{\widehat\Delta_X(\hat{\underline{t}})>0\}$
is bounded by $1$,
\begin{equation}
  |R_{1n}|
  \le
  \sup_{t\in\mathcal T_n}
  \left|
    \mathbb P_n\!\left[
      \widehat\Gamma_\Delta^{\,t}(O)
      -\Gamma_\Delta^{\,t}(O)
    \right]
  \right|.
\end{equation}
The score $\Gamma_\Delta^{\,t}$ is Neyman-orthogonal (Lemma~\ref{lem:orthogonal_pseudo_outcome}): its pathwise derivative
with respect to each nuisance is zero at the truth, so the leading-order
sensitivity to nuisance errors cancels. What remains is a second-order term
that, after cross-fitting, can be bounded by products of $L_2$ estimation
errors. Specifically, expanding $\widehat\Gamma_\Delta^{\,t}-\Gamma_\Delta^{\,t}$
in each nuisance component and applying the Cauchy--Schwarz inequality yields
terms of the form
$\|\widehat\pi_a-\pi_a\|_2\|\widehat\mu_a-\mu_a\|_2$ and
$\|\widehat g_{a'}-g_{a'}\|_2\|\widehat\mu_a-\mu_a\|_2$,
each $o_p(n^{-1/2})$ by DML conditions.
Cross-fitting ensures these products are evaluated on an independent fold,
so no empirical process complexity is incurred.
Hence $R_{1n}=o_p(n^{-1/2})$.

\medskip
\noindent\textbf{Bounding $R_{2n}$: switch-change remainder.}
The remainder $R_{2n}$ arises because the estimated indicator
$\mathbbm 1\{\widehat\Delta_X(\hat{\underline{t}})>0\}$
may differ from the oracle indicator
$\mathbbm 1\{\Delta_X(\underline{t})>0\}$.
When the two indicators agree, the summand is zero; when they disagree,
the summand is $\pm\Gamma_\Delta^{\,\hat{\underline{t}}}(O)$.
Therefore
\[
  |R_{2n}|
  \le
  \mathbb P_n\!\left[
    \left|
      \mathbbm 1\{\widehat\Delta_X(\hat{\underline{t}})>0\}
      -\mathbbm 1\{\Delta_X(\underline{t})>0\}
    \right|
    \left|\Gamma_\Delta^{\,\hat{\underline{t}}}(O)\right|
  \right].
\]
By the law of large numbers and condition~\eqref{eq:switch_stability}, the population analogue of the 
right-hand side is $o_p(n^{-1/2})$, and the empirical average concentrates 
around it at rate $O_p(n^{-1/2})$. Hence $R_{2n} = o_p(n^{-1/2})$.

\medskip
\noindent\textbf{Bounding $R_{3n}$: maximiser-shift remainder.}
This remainder captures the effect of evaluating the oracle score at the
estimated maximiser $\hat{\underline{t}}$ rather than at the true maximiser
$\underline{t}$.
Even with perfect nuisance estimation, using $\hat{\underline{t}}$ introduces
a bias because $\hat{\underline{t}}$ is random and correlated with the data
used to form the score. We split $R_{3n}$ into a deterministic population
part and a stochastic centred part:
\begin{equation}
  R_{3n}
  =
  \underbrace{
    P_0\!\left[
      \mathbbm 1\{\Delta_X(\underline{t})>0\}
      \Bigl(
        \Gamma_\Delta^{\,\hat{\underline{t}}}-\Gamma_\Delta^{\,\underline{t}}
      \Bigr)
    \right]
  }_{B_n}
  +
  \underbrace{
    (\mathbb P_n-P_0)\!\left[
      \mathbbm 1\{\Delta_X(\underline{t})>0\}
      \Bigl(
        \Gamma_\Delta^{\,\hat{\underline{t}}}-\Gamma_\Delta^{\,\underline{t}}
      \Bigr)
    \right]
  }_{S_n}.
\end{equation}

\emph{Population part $B_n$.}
Since 
$\mathbb{E}[\Gamma_\Delta^t(O)\mid X] = \Delta_X(t)$ for any fixed $t$ 
(Corollary~\ref{cor:delta-pseudo-outcome}),
\[
  B_n = \mathbb{E}\!\left[\mathbbm 1\{\Delta_X(\underline{t})>0\}
         \bigl(\Delta_X(\hat{\underline{t}}) - \Delta_X(\underline{t})\bigr)\right].
\] 
This is a deterministic value-loss bias: it measures how much the bound
objective $\Delta_X(\cdot)$ deteriorates when the estimated maximiser
replaces the oracle. 
Since $\underline{t}$ maximises $\Delta_X(\cdot)$, the integrand is 
non-positive and
\[
  |B_n| \le \mathbb{E}\!\left[
    \bigl|\Delta_X(\hat{\underline{t}}) - \Delta_X(\underline{t})\bigr|
  \right] = o_p(n^{-1/2}),
\]
by condition~\eqref{eq:value_loss}.

\emph{Stochastic part $S_n$.}
Because
$\hat{\underline{t}}$ is computed on an independent cross-fitting fold, the
score difference
$\Gamma_\Delta^{\,\hat{\underline{t}}}-\Gamma_\Delta^{\,\underline{t}}$
is a function indexed by a random but asymptotically stable index.
Therefore $S_n = o_p(n^{-1/2})$ by condition~\eqref{eq:equicont}.
Hence $R_{3n}=B_n+S_n=o_p(n^{-1/2})$.

Therefore,
\begin{equation}
\widehat{\underline{\theta}}-\underline{\theta}
=
(\mathbb P_n-P_0)\phi_{\underline{\theta}}
+
o_p(n^{-1/2}).
\end{equation}
Equivalently,
\begin{equation}
\sqrt n(\widehat{\underline{\theta}}-\underline{\theta})
=
\frac{1}{\sqrt n}\sum_{i=1}^n
\phi_{\underline{\theta}}(O_i)
+
o_p(1).
\end{equation}
By the central limit theorem,
\begin{equation}
\sqrt n(\widehat{\underline{\theta}}-\underline{\theta})
\overset{d}{\longrightarrow}
\mathcal N(0,\mathbb V\{\phi_{\underline{\theta}}(O)\}).
\end{equation}

The same argument for the upper bound gives
\begin{equation}
\sqrt n(\widehat{\overline{\theta}}-\overline{\theta})
=
\frac{1}{\sqrt n}\sum_{i=1}^n
\phi_{\overline{\theta}}(O_i)
+
o_p(1).
\end{equation}
Consistency of the empirical variance estimator follows from $L_2$ consistency of the estimated influence functions:
\begin{equation}
\mathbb P_n
\left[
\{\widehat\phi_{\overline{\underline{\theta}}}(O)
-
\phi_{\overline{\underline{\theta}}}(O)\}^2
\right]
=o_p(1).
\end{equation}
Thus,
\begin{equation}
\frac{
\widehat{\overline{\underline{\theta}}}
-
\overline{\underline{\theta}}
}{
\widehat{\mathrm{se}}(\widehat{\overline{\underline{\theta}}})
}
\overset{d}{\longrightarrow}
\mathcal N(0,1).
\end{equation}
\end{proof}

\newpage

\section{Experimental Details}
\label{app:experimental_details}

This appendix provides the details needed to reproduce the experiments in Section~\ref{sec:experiments}. Appendix~\ref{app:dgp} specifies the DGPs, Appendix~\ref{app:oracle} derives the oracle quantities against which the
estimators are benchmarked, Appendix~\ref{app:additional_results} reports
results not shown in the main text, and 
Appendix~\ref{app:implementation} describes the model
architectures.

\subsection{Data-generating process (DGP)}
\label{app:dgp}
 
We benchmark the estimators on five data-generating processes that differ along three axes, namely the dimension of the covariate vector, the dependence structure among covariates, and the functional form of the nuisance functions. Three designs with \(3\), \(5\), and \(10\) covariates are used in Section~\ref{sec:experiments}, and two designs with nonlinear treatment, mediator, and outcome models are used in the additional experiments of Appendix~\ref{app:additional_results}. 

\textbf{Components.}
In every design the treatment and mediator are drawn from
\begin{equation}
\label{eq:propensities}
\pi(x)=\mathrm{expit}\{\alpha_A+\eta_A(x)\},
\qquad
p_a(x)=\mathrm{expit}\{\alpha_M+\beta_M a+\eta_M(x)\},
\end{equation}
and the potential outcomes are
\begin{equation}
\label{eq:outcome}
Y(a,m)=\eta_Y(X)+d(X_1)a+q(X_1)m+c\,am+\sigma_y U_b+\sigma_m m U_m,
\end{equation}
with \(U_b,U_m\sim N(0,1)\) independent of each other and of \(X\). The direct and
mediator effects are heterogeneous in \(X_1\),
\begin{equation}
\label{eq:dq}
d(x_1)=d_L+(d_R-d_L)\,\mathrm{expit}(kx_1),
\qquad
q(x_1)=q_L+(q_R-q_L)\,\mathrm{expit}(kx_1),
\end{equation}
with \(d_L>0>d_R\) and \(q_L<0<q_R\), so that treatment helps part of the population and
harms another part along each pathway. This is the regime in which path-specific FNA
carries information beyond the average effect, and it is the feature we hold fixed across
designs so that differences in estimator performance are attributable to dimension,
dependence, and functional form rather than to a change in the target. Counterfactual
mediators are generated from a uniform draw \(U_M\sim\mathrm{Unif}(0,1)\) shared across
arms,
\begin{equation}
\label{eq:mediator_coupling}
M(a')=\mathbbm 1\{U_M<p_{a'}(X)\},
\qquad
M=AM(1)+(1-A)M(0),
\end{equation}
and \(\beta_M>0\) gives \(p_1(x)>p_0(x)\) for every \(x\), hence \(M(1)\ge M(0)\) almost
surely. Table~\ref{tab:dgp_params} lists the parameters of
Eqs.~\eqref{eq:propensities}--\eqref{eq:dq}, which are common to all five designs.
 
\begin{table}[h]
\centering
\caption{\textbf{Parameter values common to all designs.}}
\label{tab:dgp_params}
\footnotesize
\begin{tabular}{llc}
\toprule
\textbf{Component} & \textbf{Parameter} & \textbf{Value} \\
\midrule
Treatment model
& \(\alpha_A\) intercept & \(0.20\) \\
\midrule
\multirow{2}{*}{Mediator model}
& \(\alpha_M\) intercept & \(-0.90\) \\
& \(\beta_M\) treatment effect on mediator & \(2.00\) \\
\midrule
\multirow{5}{*}{Outcome model}
& \(d_L,\,d_R\) direct effect limits & \(3.00,\ -0.75\) \\
& \(q_L,\,q_R\) mediator effect limits & \(-1.10,\ 1.20\) \\
& \(k\) transition sharpness & \(6.0\) \\
& \(c\) treatment-mediator interaction & \(0.10\) \\
& \(\sigma_y,\,\sigma_m\) noise scales & \(0.25,\ 0.15\) \\
\bottomrule
\end{tabular}
\end{table}

\subsubsection{Designs used in the main text}
\label{app:dgp_main}
 
\textbf{Three covariates.}
Two continuous covariates and one binary covariate that depends on them,
\begin{equation}
X_1,X_2\sim N(0,1),
\qquad
X_3\mid X_1,X_2\sim\mathrm{Bern}\{\mathrm{expit}(0.40X_1-0.30X_2)\},
\end{equation}
with \(X_1\) and \(X_2\) drawn independently.

\textbf{Five covariates.}
Three continuous and two binary covariates, drawn independently,
\begin{equation}
X_1,X_2,X_4\sim N(0,1),
\qquad
X_3,X_5\sim\mathrm{Bern}(0.5).
\end{equation}
This design raises the dimension of the nuisance regressions while keeping the covariate
law simple, which isolates dimension from dependence.
 
\textbf{Ten covariates.}
A layered design in which most covariates are generated from earlier ones, so that the
covariate law carries substantial dependence. With \(\varepsilon_j\sim N(0,1)\)
independent across \(j\),
\begin{equation}
\begin{aligned}
&X_1,X_2\sim N(0,1),
&&X_5\sim\mathrm{Bern}(0.5),\\
&X_3\sim\mathrm{Bern}\{\mathrm{expit}(0.40X_1-0.30X_2)\},
&&X_4=0.50X_1+0.30X_2+\varepsilon_4,\\
&X_6=0.60X_3-0.40X_5+\varepsilon_6,
&&X_7=0.30X_4+0.20X_6+\varepsilon_7,\\
&X_8\sim\mathrm{Bern}\{\mathrm{expit}(0.50X_2+0.40X_7)\},
&&X_9=0.40X_1-0.30X_8+\varepsilon_9,\\
&X_{10}\sim\mathrm{Bern}\{\mathrm{expit}(0.30X_4-0.50X_9)\}.
&&
\end{aligned}
\end{equation}
The dependence among covariates makes the nuisance regressions harder to fit but does not
enter the oracle quantities, since all of them condition on the realised \(X\).
 
For these three designs the index functions are linear, and
Table~\ref{tab:dgp_coefs} gives their coefficients.
 
\begin{table}[h]
\centering
\caption{\textbf{Index-function coefficients for the designs of
Section~\ref{sec:experiments}.} Entries are the coefficients of
\(\eta_A\), \(\eta_M\), and \(\eta_Y\) in Eqs.~\eqref{eq:propensities}
and~\eqref{eq:outcome}. A dash marks a covariate not present in the design.}
\label{tab:dgp_coefs}
\footnotesize
\setlength{\tabcolsep}{4pt}
\begin{tabular}{lccccccccc}
\toprule
& \multicolumn{3}{c}{\textbf{3 covariates}}
& \multicolumn{3}{c}{\textbf{5 covariates}}
& \multicolumn{3}{c}{\textbf{10 covariates}} \\
\cmidrule(lr){2-4}\cmidrule(lr){5-7}\cmidrule(lr){8-10}
& \(\eta_A\) & \(\eta_M\) & \(\eta_Y\)
& \(\eta_A\) & \(\eta_M\) & \(\eta_Y\)
& \(\eta_A\) & \(\eta_M\) & \(\eta_Y\) \\
\midrule
\(X_1\)    & \(0.35\)  & \(0.15\)  & \(0.10\)  & \(0.35\)  & \(0.15\)  & \(0.10\)  & \(0.35\)  & \(0.15\)  & \(0.10\)  \\
\(X_2\)    & \(-0.25\) & \(0.40\)  & \(0.30\)  & \(-0.25\) & \(0.40\)  & \(0.30\)  & \(-0.20\) & \(0.30\)  & \(0.25\)  \\
\(X_3\)    & \(0.50\)  & \(-0.30\) & \(-0.20\) & \(0.50\)  & \(-0.30\) & \(-0.20\) & \(0.40\)  & \(-0.25\) & \(-0.15\) \\
\(X_4\)    & --        & --        & --        & \(-0.15\) & \(0.20\)  & \(0.25\)  & \(-0.15\) & \(0.20\)  & \(0.20\)  \\
\(X_5\)    & --        & --        & --        & \(0.30\)  & \(-0.10\) & \(-0.15\) & \(0.25\)  & \(-0.10\) & \(-0.10\) \\
\(X_6\)    & --        & --        & --        & --        & --        & --        & \(-0.10\) & \(0.15\)  & \(0.15\)  \\
\(X_7\)    & --        & --        & --        & --        & --        & --        & \(0.30\)  & \(-0.20\) & \(-0.20\) \\
\(X_8\)    & --        & --        & --        & --        & --        & --        & \(-0.20\) & \(0.25\)  & \(0.10\)  \\
\(X_9\)    & --        & --        & --        & --        & --        & --        & \(0.15\)  & \(-0.15\) & \(0.25\)  \\
\(X_{10}\) & --        & --        & --        & --        & --        & --        & \(-0.25\) & \(0.10\)  & \(-0.15\) \\
\bottomrule
\end{tabular}
\end{table}
 
\subsubsection{Designs with nonlinear nuisance functions}
\label{app:dgp_nonlinear}
 
The two designs below use the covariate law of the three-covariate setting and replace the
linear index functions by nonlinear ones, in the treatment, mediator, and outcome models
simultaneously. The polynomial design introduces curvature and interaction, and the
sinusoidal design introduces bounded oscillation, so the two stress the nuisance learners
in different ways. Table~\ref{tab:dgp_nonlinear} gives the index functions and
Appendix~\ref{app:additional_results} reports the results.
 
\begin{table}[h]
\centering
\caption{\textbf{Index functions for the nonlinear designs.} The covariate law is that of
the three-covariate design, and all parameters in Table~\ref{tab:dgp_params} are
unchanged.}
\label{tab:dgp_nonlinear}
\footnotesize
\begin{tabular}{ll}
\toprule
\textbf{Design} & \textbf{Index functions} \\
\midrule
\multirow{3}{*}{Polynomial}
& \(\eta_A(x)=0.35x_1-0.25x_2+0.50x_3+0.15x_1^2-0.10x_2^2+0.20x_1x_2\) \\
& \(\eta_M(x)=0.15x_1+0.40x_2-0.30x_3-0.10x_1^2+0.20x_1x_2\) \\
& \(\eta_Y(x)=0.10x_1+0.30x_2-0.20x_3+0.15x_1^2-0.10x_1x_2\) \\
\midrule
\multirow{3}{*}{Sinusoidal}
& \(\eta_A(x)=0.55\sin x_1-0.40\cos x_2+0.50x_3\) \\
& \(\eta_M(x)=0.35\sin x_1+0.25\cos x_2-0.30x_3\) \\
& \(\eta_Y(x)=0.25\sin x_1+0.35\cos x_2-0.20x_3\) \\
\bottomrule
\end{tabular}
\end{table}
 
\subsection{Oracle quantities}
\label{app:oracle}
 
The quantities below are available in closed form up to an expectation over \(X\), and the
derivations use only \(d\), \(q\), \(p_a\), and conditional Gaussianity of the outcome, so
they apply to every design of Appendix~\ref{app:dgp}.
 
\textbf{Nested potential outcomes.}
Substituting Eq.~\eqref{eq:mediator_coupling} into Eq.~\eqref{eq:outcome},
\begin{equation}
\begin{aligned}
Y_{00}&:=Y(0,M(0))=\eta_Y(X)+q(X_1)M(0)+\sigma_yU_b+\sigma_mM(0)U_m,\\
Y_{10}&:=Y(1,M(0))=\eta_Y(X)+d(X_1)+\{q(X_1)+c\}M(0)+\sigma_yU_b+\sigma_mM(0)U_m,\\
Y_{11}&:=Y(1,M(1))=\eta_Y(X)+d(X_1)+\{q(X_1)+c\}M(1)+\sigma_yU_b+\sigma_mM(1)U_m.
\end{aligned}
\end{equation}
The pathway contrasts simplify to
\begin{equation}
\label{eq:contrasts}
Y_{10}-Y_{00}=d(X_1)+cM(0),
\qquad
Y_{11}-Y_{10}=\{M(1)-M(0)\}\{q(X_1)+c+\sigma_mU_m\},
\end{equation}
so the baseline \(\eta_Y\) and the shared noise \(U_b\) cancel from both contrasts. This
is what keeps the true FNA comparable across designs even though \(\eta_Y\) changes.
 
\textbf{Average effects.}
The natural direct, natural indirect, and total average effects are
\begin{equation}
\mathrm{NDE}=\E[d(X_1)+c\,p_0(X)],
\qquad
\mathrm{NIE}=\E[\{p_1(X)-p_0(X)\}\{q(X_1)+c\}],
\qquad
\mathrm{ATE}=\mathrm{NDE}+\mathrm{NIE}.
\end{equation}
 
\textbf{True path-specific FNA.}
Since \(M(1)\ge M(0)\) almost surely, Eq.~\eqref{eq:contrasts} gives
\begin{equation}
\fna_{\mathrm{dir}}
=
\E_X\big[
\{1-p_0(X)\}\mathbbm 1\{d(X_1)<0\}
+
p_0(X)\mathbbm 1\{d(X_1)+c<0\}
\big],
\end{equation}
\begin{equation}
\fna_{\mathrm{ind}}
=
\E_X\left[
\{p_1(X)-p_0(X)\}\,
\Phi\!\left(-\frac{q(X_1)+c}{\sigma_m}\right)
\right],
\end{equation}
\begin{equation}
\begin{split}
\fna_{\mathrm{tot}}
=
\E_X\bigg[
\{1-p_1(X)\}\mathbbm 1\{d(X_1)<0\}
+
\{p_1(X)-p_0(X)\}
\Phi\!\left(-\frac{d(X_1)+q(X_1)+c}{\sigma_m}\right)
\\
+
p_0(X)\mathbbm 1\{d(X_1)+c<0\}
\bigg].
\end{split}
\end{equation}
 
\textbf{Oracle conditional CDFs.}
Because the mediator is binary and the outcome is Gaussian conditional on \((A,M,X)\),
\begin{equation}
\label{eq:oracle_cdf}
F_{a,a'}(t\mid X)
=
\{1-p_{a'}(X)\}
\Phi\!\left(
\frac{t-\mu_0(X,a)}{\sigma_y}
\right)
+
p_{a'}(X)
\Phi\!\left(
\frac{t-\mu_1(X,a)}{\sqrt{\sigma_y^2+\sigma_m^2}}
\right),
\end{equation}
where \(\mu_0(X,a)=\eta_Y(X)+d(X_1)a\) and \(\mu_1(X,a)=\mu_0(X,a)+q(X_1)+ca\). Any
dependence among the covariates drops out here, since Eq.~\eqref{eq:oracle_cdf} conditions
on the realised \(X\). These oracle CDFs are used to compute the oracle covariate-assisted
bounds reported in Tables~\ref{tab:estimation_results}
and~\ref{tab:estimation_results_nonlinear}, and are never supplied to either estimator.
 
\textbf{Numerical evaluation.}
Because the mediator is binary, the integral over \(M\) is a two-term sum and is evaluated exactly, as written in Eq.~\eqref{eq:oracle_cdf}. For continuous $M$, Monte Carlo approximation is applied. Every conditional quantity above is therefore available in closed form at each realised covariate value. The remaining expectation over \(X\) is taken as a sample
average over the \(n\) simulated covariates within each replication, so that the oracle
covariate-assisted bounds
\begin{equation}
\underline\theta=\E_X\Big[\max_{t}\{F_U(t\mid X)-F_V(t\mid X)\}\Big],
\qquad
\overline\theta=1+\E_X\Big[\min_{t}\{F_U(t\mid X)-F_V(t\mid X)\}\Big],
\end{equation}
are computed by evaluating Eq.~\eqref{eq:oracle_cdf} on the threshold grid
\(\mathcal T_n\), optimising over \(t\) pointwise in \(X\), and averaging the resulting
per-individual bounds. Reported
oracle values are averages over the \(300\) replications, and the true path-specific FNA is
computed in the same way, from the jointly simulated counterfactual outcomes
\((Y_{00},Y_{10},Y_{11})\).
 
\subsection{Additional experiments under nonlinear nuisance functions}
\label{app:additional_results}
 \begin{table}[t]
\centering
\caption{\textbf{Finite-sample ($n=5000$) performance for estimating covariate-assisted FNA bounds under nonlinear DGPs.} Results: mean over 300 simulations. Bold indicates the better-performing estimator within each pathway.}
\label{tab:estimation_results_nonlinear}
\footnotesize
\setlength{\tabcolsep}{2pt}
\begin{tabular}{clccccccl}
\toprule
\textbf{DGP}
& \textbf{Pathway}
& \textbf{FNA}
& \textbf{Oracle}
& \textbf{Estimator}
& \textbf{Mean est.}
& \textbf{Bias}
& \(\mathbf{Coverage}\)
& \textbf{Width} \\
\midrule
\multirow{6}{*}{\rotatebox[origin=c]{90}{Sinusoidal}}
& \multirow{2}{*}{Direct}
& \multirow{2}{*}{\(0.410\)}
& \multirow{2}{*}{\([0.205,\,0.450]\)}
& Plug-in
& \([0.102,\,0.507]\)
& \((-0.103,\,+0.057)\)
& \(1.000\)
& \(0.406\) \\
&
&
&
& \textbf{Orthogonal}
& \(\mathbf{[0.137,\,0.469]}\)
& \(\mathbf{(-0.068,\,+0.019)}\)
& \(\mathbf{1.000}\)
& \(\mathbf{0.332}\) \\
\addlinespace[1pt]
\cmidrule(l){2-9}
& \multirow{2}{*}{Indirect}
& \multirow{2}{*}{\(0.223\)}
& \multirow{2}{*}{\([0.183,\,0.794]\)}
& Plug-in
& \([0.024,\,0.972]\)
& \((-0.158,\,+0.179)\)
& \(1.000\)
& \(0.948\) \\
&
&
&
& \textbf{Orthogonal}
& \(\mathbf{[0.145,\,0.836]}\)
& \(\mathbf{(-0.037,\,+0.043)}\)
& \(\mathbf{1.000}\)
& \(\mathbf{0.691}\) \\
\addlinespace[1pt]
\cmidrule(l){2-9}
& \multirow{2}{*}{Total}
& \multirow{2}{*}{\(0.220\)}
& \multirow{2}{*}{\([0.070,\,0.332]\)}
& Plug-in
& \([0.092,\,0.499]\)
& \((+0.022,\,+0.168)\)
& \(0.123\)
& \(0.407\) \\
&
&
&
& \textbf{Orthogonal}
& \(\mathbf{[0.055,\,0.356]}\)
& \(\mathbf{(-0.015,\,+0.025)}\)
& \(\mathbf{1.000}\)
& \(\mathbf{0.301}\) \\
\midrule
\midrule
\multirow{6}{*}{\rotatebox[origin=c]{90}{Polynomial}}
& \multirow{2}{*}{Direct}
& \multirow{2}{*}{\(0.410\)}
& \multirow{2}{*}{\([0.226,\,0.450]\)}
& Plug-in
& \([0.131,\,0.494]\)
& \((-0.095,\,+0.044)\)
& \(1.000\)
& \(0.363\) \\
&
&
&
& \textbf{Orthogonal}
& \(\mathbf{[0.160,\,0.464]}\)
& \(\mathbf{(-0.066,\,+0.014)}\)
& \(\mathbf{1.000}\)
& \(\mathbf{0.303}\) \\
\addlinespace[1pt]
\cmidrule(l){2-9}
& \multirow{2}{*}{Indirect}
& \multirow{2}{*}{\(0.223\)}
& \multirow{2}{*}{\([0.177,\,0.803]\)}
& Plug-in
& \([0.036,\,0.952]\)
& \((-0.141,\,+0.149)\)
& \(1.000\)
& \(0.916\) \\
&
&
&
& \textbf{Orthogonal}
& \(\mathbf{[0.121,\,0.853]}\)
& \(\mathbf{(-0.056,\,+0.050)}\)
& \(\mathbf{1.000}\)
& \(\mathbf{0.731}\) \\
\addlinespace[1pt]
\cmidrule(l){2-9}
& \multirow{2}{*}{Total}
& \multirow{2}{*}{\(0.220\)}
& \multirow{2}{*}{\([0.095,\,0.337]\)}
& Plug-in
& \([0.109,\,0.473]\)
& \((+0.015,\,+0.136)\)
& \(0.213\)
& \(0.364\) \\
&
&
&
& \textbf{Orthogonal}
& \(\mathbf{[0.075,\,0.373]}\)
& \(\mathbf{(-0.020,\,+0.036)}\)
& \(\mathbf{1.000}\)
& \(\mathbf{0.298}\) \\
\bottomrule
\multicolumn{9}{p{\textwidth}}{\footnotesize \textbf{Oracle} denotes the oracle covariate-assisted bounds; \textbf{Mean est.} denotes the average estimated bounds; \textbf{FNA} is the true FNA value. \textbf{Bias} is reported as lower/upper bound bias of the estimated bounds against the oracle bounds; \textbf{Coverage} is the fraction of conservative sets covering the oracle bounds; \textbf{Width} is the mean estimated interval width. For all cases, the conservative sets covered the true FNA in 100\% of simulations.}
\end{tabular}
\end{table}
Section~\ref{sec:experiments} varies the covariate dimension and the dependence among covariates while the nuisance functions stay linear on the index scale. Here we vary the functional form instead, using the polynomial and sinusoidal designs of Appendix~\ref{app:dgp_nonlinear}. Sample size, threshold grid, network architectures, cross-fitting scheme, and tuning are unchanged from the main experiments, so the comparison isolates the effect of nonlinear nuisance
functions on estimation.
 
Because \(d\) and \(q\) in Eq.~\eqref{eq:dq} depend on \(X_1\) alone, the true path-specific FNA is close to its value in the three-covariate design, at \(0.410\), \(0.223\), and \(0.220\) for the direct, indirect, and total pathways. The oracle covariate-assisted bounds do move, since they depend on the full conditional CDFs, so each estimator is still benchmarked against the oracle of its own design.
 
\textbf{Results.}
Table~\ref{tab:estimation_results_nonlinear} reports the same metrics as Table~\ref{tab:estimation_results}, averaged over \(300\) simulations. We make three observations. (1)~The orthogonal estimator has smaller bias at both endpoints, for every pathway and under both designs. The gain is largest for the indirect pathway. (2)~The total pathway separates the two estimators most sharply. Plug-in bias is positive at both endpoints, so the estimated interval is displaced upward rather than merely widened, and coverage of the oracle bounds falls to \(0.123\) and \(0.070\) under the sinusoidal and polynomial designs. The
orthogonal estimator keeps both biases below \(0.036\) in absolute value and retains full
coverage. (3)~The orthogonal intervals are also uniformly narrower, so the bias reduction is not bought with additional
conservatism. In every cell and for both estimators, the conservative set contains the
true bounds in all \(300\) simulations. \(\Rightarrow\) \textbf{Takeaway:} \emph{the
finite-sample advantage of the orthogonal estimator persists when the
nuisance functions are nonlinear in the covariates.}
\subsection{Implementation}~\label{app:implementation}

\textbf{Algorithm.}
Algorithm~\ref{alg:orthogonal_fna} summarizes the implementation of our proposed estimator. We use \(K\)-fold cross-fitting with index sets \(\mathcal I_1,\ldots,\mathcal I_K\). For each fold \(k\), we estimate the treatment propensity \(\pi_a(X)\), the mediator model \(g_a(M\mid X)\), and the conditional outcome CDF \(\mu_a(t,M,X)\) on the observations outside \(\mathcal I_k\), and evaluate the orthogonal pseudo-outcomes \(\widehat\Gamma_{a,a'}^t(O_i)\) for \(i\in\mathcal I_k\), all \(t\in\mathcal T_n\), and \((a,a')\in\{(0,0),(1,0),(1,1)\}\).

The second-stage regression is cross-fitted with the same folds. For each fold \(k\), we regress the pseudo-outcomes of the observations outside \(\mathcal I_k\) on their covariates,
\begin{equation}
\widehat F_{a,a'}^{(-k)}(t\mid x)
=
\widehat{\mathbb E}_{-k}
\left[
\widehat\Gamma_{a,a'}^t(O)\mid X=x
\right],
\end{equation}
and use \(\widehat F_{a,a'}^{(-k)}\) only for observations in \(\mathcal I_k\). Thus, the pseudo-outcome \(\widehat\Gamma_{a,a'}^t(O_i)\) never enters the second-stage fit that is evaluated at \(X_i\). Given these estimates, for \(i\in\mathcal I_k\), we compute
\begin{equation}
\widehat\Delta_{X_i}(t)=\widehat F_U^{(-k)}(t\mid X_i)-\widehat F_V^{(-k)}(t\mid X_i),
\end{equation}
select covariate-specific grid maximizers
\begin{equation}
\hat {\underline{t}}(X_i)\in\arg\max_{t\in\mathcal T_n}\widehat\Delta_{X_i}(t),
\qquad
\hat {\overline{t}}(X_i)\in\arg\max_{t\in\mathcal T_n}\{-\widehat\Delta_{X_i}(t)\},
\end{equation}
and evaluate the one-step estimators described in Section~\ref{sec:estimation}.

\emph{Remark.} Because the pseudo-outcomes used to train \(\widehat F_{a,a'}^{(-k)}\) are themselves cross-fitted, their nuisance estimates are trained partly on \(\mathcal I_k\). The maximizers \(\hat{\underline t}(X_i)\) and \(\hat{\overline t}(X_i)\) therefore depend on \(O_i\) only indirectly, through first-stage nuisance fits in which \(O_i\) is one of \(O(n)\) training observations, but which are never used as a second-stage training target. Hence, full independence, as in Assumption~\ref{ass:max_stability}, can still be obtained by nested cross-fitting, in which the pseudo-outcomes used for \(\widehat F^{(-k)}\) are recomputed with nuisances trained only outside \(\mathcal I_k\), at the cost of \(K(K-1)\) nuisance fits.

\begin{algorithm}[h]
\caption{Two-stage orthogonal estimation of path-specific FNA bounds}
\label{alg:orthogonal_fna}
\DontPrintSemicolon
\KwIn{Observed data $\{O_i=(X_i,A_i,M_i,Y_i)\}_{i=1}^n$, threshold grid $\mathcal T_n$, number of folds $K$.}
\KwOut{Estimates $\widehat{\underline{\theta}},\widehat{\overline{\theta}}$, standard errors, and conservative set.}

Split $\{1,\ldots,n\}$ into folds $\mathcal I_1,\ldots,\mathcal I_K$.\;

\tcp{Stage 1: cross-fitted pseudo-outcomes}
\For{each fold $k=1,\ldots,K$}{
Fit nuisance functions $\widehat\pi_a$, $\widehat g_a$, and $\widehat\mu_a$ on $\{O_j: j\notin\mathcal I_k\}$.\;
For all $t\in\mathcal T_n$ and $(a,a')\in\{(0,0),(1,0),(1,1)\}$, compute
$\widehat\Gamma_{a,a'}^t(O_i)$ for $i\in\mathcal I_k$.\;
}

\tcp{Stage 2: cross-fitted conditional nested CDFs}
\For{each fold $k=1,\ldots,K$}{
For each $(a,a')$, fit the second-stage regression on $\{(X_j,\widehat\Gamma_{a,a'}(O_j)): j\notin\mathcal I_k\}$:
\begin{equation*}
\widehat F_{a,a'}^{(-k)}(t\mid x)
=
\widehat{\mathbb E}_{-k}\{\widehat\Gamma_{a,a'}^t(O)\mid X=x\}.
\end{equation*}
For $i\in\mathcal I_k$, compute
$\widehat\Delta_{X_i}(t)=\widehat F_U^{(-k)}(t\mid X_i)-\widehat F_V^{(-k)}(t\mid X_i)$ and the grid maximisers
\begin{equation*}
\hat {\underline{t}}(X_i)\in\arg\max_{t\in\mathcal T_n}\widehat\Delta_{X_i}(t),
\qquad
\hat {\overline{t}}(X_i)\in\arg\max_{t\in\mathcal T_n}\{-\widehat\Delta_{X_i}(t)\}.
\end{equation*}
}

\tcp{One-step estimation and inference}
Compute one-step estimates:
\begin{equation*}
\widehat{\underline{\theta}}
=
\mathbb P_n\left[
\mathbbm 1\{\widehat\Delta_X(\hat {\underline{t}}(X))>0\}
\widehat\Gamma_\Delta^{\hat {\underline{t}}(X)}(O)
\right],
\qquad
\widehat{\overline{\theta}}
=
\mathbb P_n\left[
1+
\mathbbm 1\{-\widehat\Delta_X(\hat {\overline{t}}(X))>0\}
\widehat\Gamma_\Delta^{\hat {\overline{t}}(X)}(O)
\right].
\end{equation*}

Estimate influence functions:
\begin{equation*}
\widehat\phi_{\underline{\theta}}(O)
=
\mathbbm 1\{\widehat\Delta_X(\hat {\underline{t}}(X))>0\}
\widehat\Gamma_\Delta^{\hat {\underline{t}}(X)}(O)
-
\widehat{\underline{\theta}},
\qquad
\widehat\phi_{\overline{\theta}}(O)
=
1+
\mathbbm 1\{-\widehat\Delta_X(\hat {\overline{t}}(X))>0\}
\widehat\Gamma_\Delta^{\hat {\overline{t}}(X)}(O)
-
\widehat{\overline{\theta}}.
\end{equation*}

Calculate standard errors
\begin{equation*}
\widehat{\mathrm{se}}(\widehat{\underline{\theta}})
=
\sqrt{\mathbb V_n(\widehat\phi_{\underline{\theta}})/n},
\qquad
\widehat{\mathrm{se}}(\widehat{\overline{\theta}})
=
\sqrt{\mathbb V_n(\widehat\phi_{\overline{\theta}})/n}.
\end{equation*}

Return the conservative set for FNA
\begin{equation*}
    \left[
\max\{0,\ \widehat{\underline{\theta}}-z_{1-\alpha/2}\widehat{\mathrm{se}}(\widehat{\underline{\theta}})\},
\ 
\min\{1,\ \widehat{\overline{\theta}}+z_{1-\alpha/2}\widehat{\mathrm{se}}(\widehat{\overline{\theta}})\}
\right].
\end{equation*}
\end{algorithm}

The plug-in estimator uses the same nuisance components to estimate \(F_{a,a'}(t\mid X)\) through the mediation g-formula, but then directly plugs the estimated conditional CDFs into the conditional Makarov bounds without the final orthogonal score correction.

\textbf{Model architecture and training details.}

The propensity nuisance \(p(A\mid X)\) and mediator nuisance \(p(M\mid A,X)\) are estimated using shallow neural-network binary classifiers. The propensity model takes \(X\) as input and predicts \(A\), while the mediator model takes \((A,X)\) as input and predicts \(M\). Both models are trained using binary cross-entropy loss.

For the outcome CDF nuisance, $\mu_a(t,M,X)=\Pr(Y\le t\mid A=a,M,X)$,

we use a neural conditional CDF model trained with the continuous ranked probability score (CRPS). The model takes \((A,M,X)\) as input and outputs CDF values on the threshold grid
\(\mathcal T_n=\{t_1,\ldots,t_T\}\). Monotonicity in \(t\) is enforced architecturally, which yields a nondecreasing CDF estimate over the grid. The CRPS loss is approximated by the trapezoidal rule,
\begin{equation}
\mathcal L_{\mathrm{CRPS}}^{\mu}
=
\frac{1}{n}
\sum_{i=1}^n
\sum_{k=1}^{T}
w_k
\left\{
\widehat \mu(A_i,M_i,X_i;t_k)
-
\mathbbm 1(Y_i\le t_k)
\right\}^2,
\end{equation}
where \(w_k\) are trapezoidal integration weights over the threshold grid. This objective is a proper scoring rule for conditional distributions and targets the full conditional CDF.

For the second-stage regression, $X\mapsto F_{a,a'}(t\mid X)$,

we use the same CRPS-style conditional CDF architecture, but train it on pseudo-outcome vectors rather than binary threshold indicators. Specifically, for each observation \(i\), the second-stage target is a pseudo-outcome vector
\begin{equation}
\widehat\Gamma_i
=
\left(
\widehat\Gamma_i(t_1),\ldots,\widehat\Gamma_i(t_T)
\right),
\end{equation}
and the model jointly predicts
\begin{equation}
\widehat F_{a,a'}(t_1\mid X_i),\ldots,
\widehat F_{a,a'}(t_T\mid X_i).
\end{equation}
The implemented pseudo-outcome CRPS loss is
\begin{equation}
\mathcal L_{\mathrm{CRPS}}^{a,a'}
=
\frac{1}{n}
\sum_{i=1}^n
\sum_{k=1}^{T}
w_k
\left\{
\widehat F_{a,a'}(t_k\mid X_i)
-
\widehat\Gamma_i(t_k)
\right\}^2.
\end{equation}
Thus, the second-stage regression estimates the entire conditional CDF jointly across thresholds, rather than fitting separate regressions at each \(t_k\). 

Unless otherwise stated, all neural networks are trained with Adam using a fixed learning rate, weight decay, batch size, and number of epochs. Hyperparameters are held fixed across simulation repetitions and methods.

\textbf{Compute resources}
The implementation is written in PyTorch. PyTorch automatically uses a CUDA GPU when available; otherwise the experiments can be run on CPU with longer runtime.

The main computational cost comes from fitting nuisance models across cross-fitting folds, evaluating pseudo-outcomes over the threshold grid, and training the second-stage conditional CDF regressions. The memory requirement scales as $O(n|\mathcal T_n|)$
for storing pseudo-outcomes and conditional CDF estimates for each nested CDF pair. Holding the network architecture fixed, the dominant training cost scales approximately as
$O(K \times n \times |\mathcal T_n| \times \text{epochs})$,
up to constants depending on the neural-network architecture.

For the binary-mediator DGP, mediator integration is exact and does not require Monte Carlo sampling. For continuous mediators, the same implementation can use Monte Carlo integration, in which case the cost additionally scales with the number of mediator draws.

The reported experiments can be reproduced on a single standard GPU or on a modern multi-core CPU. We did not use large-scale distributed training. Preliminary experiments for debugging the DGP, checking oracle calculations, and validating the CDF learner used the same codebase and comparable per-run resources, but are not needed to reproduce the reported results.

\section{Practical Considerations}
\label{app:practical-considerations}

Our work is primarily foundational. We study how harm can be formalized, decomposed, and partially identified when treatment effects operate through multiple causal pathways. The proposed framework is intended to clarify what can and cannot be learned from data about path-specific harm, rather than to provide an off-the-shelf decision system for immediate deployment.

If such methods are used in practice, they should be applied cautiously and in settings where the causal assumptions, data quality, and domain-specific consequences of decisions can be carefully assessed. In particular, we recommend that practitioners prioritize reliability, ethical oversight, and safe use. Estimates of harm bounds should be interpreted as decision-support tools, not as definitive guarantees about individual outcomes. Since the relevant quantities involve counterfactual comparisons that are never observable, practical conclusions should account for both statistical uncertainty and the inherent ambiguity captured by partial identification.

We also emphasize that there is often no single universally correct notion of harm. Different applications may require different normative choices: for example, whether harm should be evaluated in aggregate, at the individual level, along specific causal pathways, or relative to particular baseline interventions. These choices may lead to different conclusions, not because one notion is necessarily right and another wrong, but because they encode different ethical and practical priorities. Our framework is therefore best viewed as a way to make these choices explicit and to analyze their implications transparently.

Before deployment in high-stakes domains such as healthcare, social policy, or automated decision-making, practitioners should engage domain experts, affected stakeholders, and ethics reviewers to determine which harm notion is appropriate, whether the required assumptions are credible, and how conservative the resulting decision rule should be. When uncertainty is substantial, we recommend erring on the side of caution, reporting sensitivity analyses, and avoiding automated use without human oversight.

\end{document}